%% file: main.tex
\documentclass[nohyperref]{article}

\usepackage{microtype}
\usepackage{graphicx}
\usepackage{booktabs} 

\usepackage{hyperref}

\usepackage[accepted]{icml2022}

\usepackage{amsmath}
\usepackage{amssymb}
\usepackage{mathtools}
\usepackage{amsthm}

\usepackage{bbm}
\usepackage{booktabs}
\newcommand{\ra}[1]{\renewcommand{\arraystretch}{#1}}

\usepackage{subcaption}
\usepackage[export]{adjustbox}
\usepackage{array}
\usepackage{xcolor}

\usepackage{wrapfig}
\usepackage{float}
\usepackage{placeins}

\usepackage[capitalize,noabbrev]{cleveref}

\theoremstyle{plain}
\newtheorem{theorem}{Theorem}[section]

\theoremstyle{definition}

\theoremstyle{remark}
\newtheorem{remark}[theorem]{Remark}

\usepackage[textsize=tiny]{todonotes}

\icmltitlerunning{Generative Models with Information-Theoretic Protection Against Membership Inference Attacks}

\begin{document}

\twocolumn[
\icmltitle{Generative Models with Information-Theoretic Protection\\ Against Membership Inference Attacks}




\begin{icmlauthorlist}
\icmlauthor{Parisa Hassanzadeh}{jpmc}
\icmlauthor{Robert E. Tillman}{ie}
\end{icmlauthorlist}

\icmlaffiliation{jpmc}{J.P. Morgan AI Research, San Francisco, CA, USA.}
\icmlaffiliation{ie}{Index Exchange, New York, NY, USA. This work was done when Robert E. Tillman was with J.P. Morgan AI Research.}

\icmlcorrespondingauthor{parisa.hassanzadeh}{parisa.hassanzadeh@jpmorgan.com}
\icmlcorrespondingauthor{Robert Tillman}{ rtillman@alumni.cmu.edu}

\icmlkeywords{deep generative models, adversarial learning & robustness, privacy-aware machine learning}

\vskip 0.3in
]


 

\printAffiliationsAndNotice{}  

\providecommand{\keywords}[1]
{
  \small	
  \textbf{\textit{Keywords---}} #1
}

\begin{abstract}
Deep generative models, such as Generative Adversarial Networks (GANs), synthesize diverse high-fidelity data samples by estimating the underlying distribution of high dimensional data. Despite their success, GANs may disclose private information from the data they are trained on, making them susceptible to adversarial attacks such as membership inference attacks, in which an adversary aims to determine if a record was part of the training set. We propose an information theoretically motivated regularization term that prevents the generative model from overfitting to training data and encourages generalizability. We show that this penalty minimizes the Jensen–Shannon divergence  between  components of the generator trained on data with different membership, and that it can be implemented at low cost using an additional classifier.  Our experiments on image datasets demonstrate that with the proposed regularization, which comes at only a small added computational cost, GANs are able to preserve privacy and generate high-quality samples that achieve better downstream classification performance  compared to  non-private and differentially private generative models.
\end{abstract}


\section{Introduction}
Generative models for synthetic data are promising approaches addressing the need for large quantities of (often sensitive) data for training machine learning models. This need is especially  pronounced  in  domains with strict privacy-protecting regulations, such as healthcare and finance, as well as domains with data scarcity issues or where collecting data is expensive, such as autonomous driving. Given a set of training samples, generative models \cite{hinton2006reducing,kingma2013auto,goodfellow2014generative} approximate the data generating distribution from which new samples can be generated.

Deep generative models such as Generative Adversarial Networks (GANs) and Variational Autoencoders (VAEs) have shown great potential for synthesizing  samples with high-fidelity \cite{arjovsky2017wasserstein,radford2015unsupervised,brock2018large}, with various applications in image super-resolution, image-to-image translation, object detection and text-to-image synthesis. However, recent studies \cite{hayes2019logan,hilprecht2019monte,chen2020gan} have shown that generative models can leak sensitive information from the training data, making them vulnerable to privacy attacks. For example, Membership Inference Attacks (MIA), which aim to infer whether a given data record was used for training the model, and active inference attacks in collaborative settings, which reconstruct training samples from the generated ones, were shown to be very successful in \cite{hayes2019logan} and \cite{hitaj2017deep}, respectively. Several paradigms have been proposed for defending against such attacks, including differentially private  mechanisms that ensure a specified level of privacy protection for the training records \cite{DPGAN,jordon2018pate,liu2019ppgan}. Other defense frameworks prevent the memorization of training data by using regularization such as weight normalization and dropout training \cite{hayes2019logan}, or with adversarial regularization by using an internal privacy discriminator as in  \cite{mukherjee2021privgan}.

\paragraph{Contributions:} 
In this paper, we focus on deep generative models and propose a new mechanism to provide privacy protection against membership inference attacks.  Specifically, our contributions are as follows:
\begin{itemize}
    \item We propose a modification to the GAN objective which encourages learning more generalizable representations that are less vulnerable to MIAs. To prevent memorization of the training set, we train the generator on different subsets of the data while penalizing learning which  subset a given training instance is from. This penalty is quantified by the mutual information between the generated samples and a latent code that represents the subset membership of training samples.
    \item We show that the proposed information-theoretic regularization is equivalent to minimizing the divergence between the generative  distribution learned from training on different subsets of data and can be implemented using a classifier that interacts with the generator.  
    \item We demonstrate that the proposed privacy-preserving mechanism requires less training data and has lower computational cost compared to previously proposed approaches.
    \item We empirically evaluate our proposed model on benchmark image 
    datasets (Fashion-MNIST, CIFAR-10), and demonstrate its effective defense against MIAs without significant compromise in generated sample quality and downstream task performance. We compare the privacy-fidelity trade-off of our proposed model to non-private models as well as other privacy-protecting mechanisms, and show that our model achieves better trade-offs with negligible added computational cost compared to its non-private counterparts.
\end{itemize}

 \section{Background}
\subsection{Generative Adversarial Networks}
GANs train deep generative models through a minimax game between a generative model $G$ and a discriminative model $D$. The generator learns a mapping $G(\mathbf{z})$ from a prior noise distribution $p_{\mathbf{z}}(z)$ to the data space, such that the generator distribution $p_{g}(x)$ matches the data distribution $p_{\mathbf{x}}(x)$. The  discriminator is trained to correctly distinguish between data samples and synthesized samples, and assigns a value $D(\mathbf{x})$ representing its confidence that sample $\mathbf{x}$ came from the training data rather than the generator. $G$ and $D$ optimize the following objective function
\begin{align}
    \min_{G} &\max_{D}   V(D,G) \coloneqq\notag \\ &\mathbb{E}_{\mathbf{x}\sim p_{\mathbf{x}}} [\log(D(\mathbf{x}))] 
  + \mathbb{E}_{\mathbf{z}\sim p_{\mathbf{z}}} [\log(1- D(G(\mathbf{z})))]. \label{eq:opt}
\end{align}
For a given generator $G$, the optimal discriminator is given by $D^\star({x}) = \frac{p_{\mathbf{x}}({x})}{p_{\mathbf{x}}({x}) + p_{g}({x})}$. Training a GAN minimizes the divergence between the generated and real data distributions \cite{goodfellow2014generative}.

\subsection{Membership Inference Attacks} 
MIAs are privacy attacks on trained machine learning models where the goal of the attacker, who may have limited or full access to the model, is to determine whether a given data point was used for training the model. Based on the extent of information available to an attacker, MIAs are divided into various categories, such as black-box and white-box attacks. \cite{shokri2017membership} focuses on discriminative models in the black-box setting and shows their vulnerability to MIAs using a shadow training technique that trains a classifier to distinguish between the model predictions on members versus non-members of the training set. The authors relate the privacy leakage to  model overfitting and present mitigation strategies against MIAs by restricting the model's prediction power or by using regularization. In \cite{hayes2019logan}, several  successful attacks against GANs are proposed using models that are trained on the generated samples to detect overfitting. 
\cite{hilprecht2019monte} proposes additional MIAs against GANs and VAEs which identify members of the training set using the proximity of a given record to the  generated samples based on Monte Carlo integration. We provide a detailed description of MIAs considered in this paper in Sec.~\ref{sec:evaluation metrics}. The \textit{attack accuracy} of an MIA on a trained model is the fraction of data samples that are correctly inferred as members of the training set.

\subsection{Privacy-Preserving Generative Models} 
\emph{Privacy-preserving} generative models modify their respective model frameworks, e.g. by changing objective functions or training procedures, to reduce the effectiveness of MIAs and other privacy attacks. For example, differentially private GANs \cite{DPGAN,jordon2018pate,liu2019ppgan} provide formal membership privacy guarantees by adding carefully designed noise during the training process. Differentially private generative models, however, often result in low-fidelity  samples unless a low-privacy setting is used. Another approach is adversarial regularization, such as in \cite{mukherjee2021privgan}, which trains multiple generator-discriminator pairs and uses a built–in privacy discriminator that acts as regularization to prevent memorization of the training set. While adversarial regularization does not provide the formal guarantees of differential privacy, \cite{mukherjee2021privgan} shows that the proposed model is able to mitigate MIAs at the cost of training multiple GANs without considerably sacrificing performance in downstream learning tasks.   

In this work, we propose a novel adversarial regularization that defends against MIAs and combats overfitting by regularizing the GAN generator loss, which we empirically show to be effective. We further show that the regularization is equivalent to minimizing the Jensen-Shannon divergence between subsets of the training dataset, resembling the objective function in \cite{mukherjee2021privgan} for the optimal discriminators and optimal privacy discriminator. The work in \cite{hoang2018mgan} uses a somewhat similar objective function in order to mitigate mode-collapse by using multiple generators and an additional classifier. However, contrary to our setup, their regularization term is maximized to encourage generators to specialize in different data modes.

\section{Privacy Preservation Through an Information-Theoretic Objective} 
In order to prevent the generative model from memorizing private information in the training data and to improve its generalization ability, we impose an information-theoretic regularization on the generator $G$. 
Let us assume that the training data is divided into $N$ non-overlapping subsets of samples $\mathbf{x}_i \sim p_{\mathbf{x}_i}(x)$, $i=1,\dots, N$. We indicate the membership of training samples from each distribution with the variable $\mathbf{c}\in\{1,\dots,N\}$, and capture the difference between sample distributions $p_{\mathbf{x}_1},\dots, p_{\mathbf{x}_N}$ through an additional latent code $\mathbf{c}$  provided to both the generator and the discriminator. Specifically, for $\mathbf{c}\sim p_{\mathbf{c}} = \text{Multinomial}(\boldsymbol{\pi})$, $\boldsymbol{\pi}= (\pi_1,\dots,\pi_N)$, where $\pi_i$ denotes the relative frequency of samples from $p_{\mathbf{x}_i}(x)$, the generator $G(\mathbf{z}, \mathbf{c})$ is provided with noise $\mathbf{z}$ and latent code  $\mathbf{c}$, and the discriminator takes a data sample  $\mathbf{x}$ and latent code  $\mathbf{c}$ as inputs. Our goal is for the generator to learn a distribution $p_g(x)$ that matches the underlying data generating distribution $p_{\mathbf{x}}(x)$, 
rather than sample distributions $\{p_{\mathbf{x}_i}(x)\}_{i=1}^N$. Therefore, we 
penalize the generator encouraging it to synthesize samples that are independent from the latent code $\mathbf{c}$, by minimizing the mutual information $I(G(\mathbf{z},\mathbf{c});\mathbf{c})$. If $G(\mathbf{z},\mathbf{c})$ and $\mathbf{c}$ are independent, then $I(G(\mathbf{z},\mathbf{c});\mathbf{c})=0$. We propose to solve the following regularized minimax game:
{\small
\begin{align}
    &\min_{G}  \max_{D}   V_{I}(D,G) \coloneqq\mathbb{E}_{(\mathbf{x},\mathbf{c})\sim p_{\mathbf{x},\mathbf{c}}} [\log(D(\mathbf{x}, \mathbf{c}))] \; + \notag \\ 
    &\quad \mathbb{E}_{\mathbf{z}\sim p_{\mathbf{z}}, \mathbf{c}\sim \boldsymbol{\pi}} [\log(1- D(G(\mathbf{z}, \mathbf{c}), \mathbf{c}))] + \lambda  I(G(\mathbf{z},\mathbf{c}); \mathbf{c}), \label{eq:opt info}
\end{align}
}%
where parameter $\lambda$ controls the trade-off between the fidelity of the generative process and its privacy protection. Let the  conditional generative distributions be denoted by $p_{g_i} (x)\coloneqq P_{\mathbf{x}|\mathbf{c}}({x}|{c=i})$, for $\mathbf{x}\sim G(\mathbf{z},\mathbf{c})$.
Then, for independently sampled $\mathbf{z}\sim p_{\mathbf{z}}$ and $\mathbf{c}\sim p_{\mathbf{c}}$, the generator distribution is $p_g = \sum_{i=1}^N \pi_i p_{g_i}$.

    
\begin{theorem}
The mutual information regularization term in \eqref{eq:opt info} is equivalent to the Jensen–Shannon divergence (JSD) between   conditional generative distributions  $p_{g_1},\dots,p_{g_N}$. 
\end{theorem}
\vspace{-0.2cm}
\begin{proof}
Let $\text{KL}(P_1||P_2)$ denote the Kullback–Leibler (KL) divergence between probability distributions $P_1$ and $P_2$. Then, from the definition of mutual information and by conditioning random variable $G(\mathbf{z},\mathbf{c})$ on $\mathbf{c}$, it follows  that  
\begin{align}
   I(G(\mathbf{z},\mathbf{c}&); \mathbf{c} )
 =\mathbb{E}_{\mathbf{c}\sim p_{\mathbf{c}}}\Big[ \text{KL} \Big(P_{G(\mathbf{z},\mathbf{c})|\mathbf{c}}({x}|{c})||P_{  G(\mathbf{z},\mathbf{c})}(x)  \Big) \Big]\notag\\
   & =  \sum\nolimits_{i=1}^N  \pi_i \text{KL} \Big(  p_{g_{i}}(x)|| p_g(x) \Big)\\
   & =  \sum\nolimits_{i=1}^N  \pi_i \text{KL} \Big(  p_{g_{i}}(x)||  \sum\nolimits_{j=1}^N \pi_j p_{g_{j}}(x)  \Big) \label{eq:jsd}\\
   & = \text{JSD}\Big(p_{g_1}(x),\dots,p_{g_N}(x)\Big),
\end{align}
where $\text{JSD}(P_1,\dots,P_N)$ denotes the JSD between distributions $P_1,\dots, P_N$.
\qedhere
\end{proof}
\vspace{-0.2cm}
 We refer to  GANs  trained with respect to \eqref{eq:opt info} as \textit{Private Information-Theoretic Generative Adversarial Networks (PIGAN)}. 

\begin{theorem}
For the optimal discriminator, the global optimum of value function  \eqref{eq:opt info} is $-\log{4}$ which is achieved if and only if $p_{\mathbf{x}_1}(x) = \dots=p_{\mathbf{x}_N}(x)=p_{g_1}(x) =\dots= p_{g_N}(x)$.  
\end{theorem}
\begin{proof}
By conditioning the data generating distribution on $\mathbf{c}$, value function $V_{I}(D,G)$ can be rewritten as 
{\small
\begin{align}
 V_{I}&(D,G) = \mathbb{E}_{\mathbf{c}\sim \boldsymbol{\pi}}\mathbb{E}_{\mathbf{x}\sim p_{\mathbf{x}_{\mathbf{c}}}} [\log(D(\mathbf{x}, \mathbf{c}))] \;+ \notag \\ 
    & \mathbb{E}_{\mathbf{c}\sim \boldsymbol{\pi}}\mathbb{E}_{\mathbf{x}\sim p_{{g_{\mathbf{c}}
    } }} [\log(1-D(\mathbf{x}, \mathbf{c}))]  
      +  \lambda \text{JSD}\Big(p_{g_1}(x),\dots,p_{g_N}(x)\Big) \notag
\end{align}
}%
Based on Proposition 1 in  \cite{goodfellow2014generative}, for a fixed generator $G(\mathbf{z},\mathbf{c})$, the  discriminator 
$D^\star(\mathbf{x},\mathbf{c}) = \frac{p_{\mathbf{x}_{\mathbf{c}}}({x})}{p_{\mathbf{x}_{\mathbf{c}}}({x}) + p_{g_{\mathbf{c}}}({x})}$ maximizes \eqref{eq:opt info}. For ${c}=\{1,\dots,N\}$, if $p_{\mathbf{x}_{{c}}}(x)=p_{g_{{c}}}(x)$,  we have $D^\star(\mathbf{x},{c}) =\frac{1}{2}$, and if $p_{g_1}(x)=\dots=p_{g_N}(x)$, then $\text{JSD}\Big(p_{g_1}(x),\dots,p_{g_N}(x)\Big)=0$. Therefore, $V_{I}(D^\star,G) = - \log{4}$.

To find the global optimum, by replacing $D^\star$ in $V_{I}(D,G)$, it follows from Theorem 1 in  \cite{goodfellow2014generative}  that
{\small
\begin{align}
    \min_{G}  V_{I}(D^\star,G) = 
    \sum\nolimits_{i=1}^N & \pi_i   \Big( -\log{4} + 2 \text{JSD} \Big( p_{\mathbf{x}_i}(x)|| p_{g_i}(x)\Big) \Big)\notag\\
    & + \lambda \text{JSD}\Big(p_{g_1}(x),\dots,p_{g_N}(x)\Big) \label{eq:opt Dstar}
\end{align}
}%
Since the JSD  between two or more distributions is always non-negative and zero
only when they are equal, the global minimum of $V_{I}(D^\star,G)$ is $-\log{4}$, which is achieved only when $p_{\mathbf{x}_1}(x) = \dots=p_{\mathbf{x}_N}(x)=p_{g_1}(x) =\dots= p_{g_N}(x)$. At this point,  the generative model conditioned on each membership code $\mathbf{c}$ perfectly replicates the data generating process.
\end{proof}

\begin{remark}
PIGAN is trained to minimize the JSD between generated sample distributions trained on different subsets of the training set. For the optimal discriminator $D^\star$, the objective function in \eqref{eq:opt Dstar} is closely related to the objective function of \textit{PrivGAN}, the privacy-preserving model proposed in \cite{mukherjee2021privgan}. It was shown in \cite{mukherjee2021privgan} that given the $N$ optimal discriminators and the optimal privacy discriminator of PrivGAN, its objective function is equivalent to the JSD between the $N$ generators' distributions,  differing from PIGAN's value function only in constant multiplicative and additive terms.
\end{remark}

\begin{remark}
The privacy penalty in PIGAN  resembles the variational regularization term used in InfoGAN \cite{chen2016infogan}. While both are based on the mutual information, the two terms are being used in completely different ways and, in fact, have opposing impacts on the GAN solution. InfoGAN \emph{maximizes} the mutual information term to learn disentangled representations in the latent space; latent codes represent learnt semantic features of the data. PIGAN, however, uses latent codes to explicitly represent membership when the data is randomly divided into different groups and \emph{minimizes} the mutual information quantity so that group membership is private.
\end{remark}
While we present the proposed privacy-preserving regularization for the GAN framework, the  method can be readily applied to other deep generative methods such as VAEs, by providing the membership latent code as additional inputs to both the encoder and decoder and minimizing the mutual information between the latent code and  samples generated by the VAE decoder.

\subsection{Implementation} 
Our goal is to minimize the mutual information term in  \eqref{eq:opt info}. By expanding the KL divergence term in \eqref{eq:jsd}, we have
\newcommand\nonumberthis{\refstepcounter{equation}\nonumber}
\begin{align}
 &\sum\nolimits_{i=1}^N  \pi_i \text{KL} \Big(  p_{g_{i}}(x)||  \sum\nolimits_{j=1}^N \pi_j p_{g_{j}}(x)  \Big) \notag\\
 &= \sum_{i=1}^N  \pi_i \mathbb{E}_{\mathbf{x}\sim p_{g_i}}\Big[ \log\Big(\frac{p_{g_i}(x)}{ \sum_{j=1}^N \pi_j p_{g_{j}}(x)}\Big)\Big]\notag\\
 & =
 \sum_{i=1}^N  \pi_i \mathbb{E}_{\mathbf{x}\sim p_{g_i}}\Big[ \log\Big(\frac{  \pi_i p_{g_i}(x)}{ \sum_{j=1}^N \pi_j p_{g_{j}}(x)}\Big)\Big] - \sum_{i=1}^N  \pi_i \log(\pi_i) 
 \label{eq:posterior}\\
 & = \sum\nolimits_{i=1}^N  \pi_i \mathbb{E}_{\mathbf{x}\sim p_{g_i}}[ \log( \widehat{\pi}_i (x) )] + H(\mathbf{c}) \label{eq:pihat} 
 \end{align}
 where \eqref{eq:posterior} follows from adding and subtracting the entropy of $\mathbf{c}$,  $H(\mathbf{c})=- \sum_{i=1}^N  \pi_i \log(\pi_i)$. In \eqref{eq:pihat}, 
 $\widehat{\pi}_i (x) \coloneqq \frac{  \pi_i p_{g_i}(x)}{ \sum_{j=1}^N \pi_j p_{g_{j}}(x)}$ denotes the  posterior probability that sample $\mathbf{x}$ belongs to  conditional generative  distribution $p_{g_{i}}$ which was trained on samples from $p_{\mathbf{x}_i}$. For fixed $\boldsymbol{\pi}$, $H(\mathbf{c})$ is constant, and the regularization term in \eqref{eq:opt info} is equivalent to the negative of the cross-entropy between distributions $\boldsymbol{\pi}$ and $\widehat{\boldsymbol{\pi}}$. 
 We estimate $\widehat{\pi}_i (x)$ using a multi-class classifier which uses the latent code $\mathbf{c}$ taken as input by $G(\mathbf{z},\mathbf{c})$ as class labels. We denote this classifier by $Q(\mathbf{x})$, and its predicted label by $\widehat{\mathbf{c}}$. The classifier is trained to correctly determine which conditional generative  distribution generated a given synthetic sample, i.e., to maximize the negative cross entropy term in \eqref{eq:pihat}, which in turn encourages the generator to synthesize data with indistinguishable conditional distributions. Therefore, \eqref{eq:opt info} becomes 
 \begin{align}
    &\min_{G}  \max_{D,Q}   V_{I}'(D,G,Q) \coloneqq\mathbb{E}_{(\mathbf{x},\mathbf{c})\sim p_{\mathbf{x},\mathbf{c}}} [\log(D(\mathbf{x}, \mathbf{c}))] \notag \\ 
    &\quad + \mathbb{E}_{\mathbf{z}\sim p_{\mathbf{z}}, \mathbf{c}\sim p_{\mathbf{c}}} [\log(1- D(G(\mathbf{z}, \mathbf{c}), \mathbf{c}))] \notag\\ 
     &\quad +\lambda  \sum\nolimits_{i=1}^N  \pi_i \mathbb{E}_{\mathbf{z}\sim p_{\mathbf{z}}}[ \log( Q(G(\mathbf{z}, i)) )]    \label{eq:opt final}
\end{align}

In our experiments, we arbitrarily partition the dataset into $N$ equal non-overlapping subsets such that  $\pi_i=\frac{1}{N}$, for which $H(\mathbf{c})=\log(N)$. For an i.i.d. dataset distributed according to $p_{\textbf{x}}(x)$, we expect that any partitioning of data will result in subsets with similar distributions. The discriminator, generator and classifier are neural networks parametrized by $\theta_d$, $\theta_g$ and $\theta_q$, respectively. Algorithm ~\ref{alg:algorithm}, presents the pseudo-code for learning the parameters by alliteratively training $G$, $D$ and $Q$ using stochastic gradient descent.  In order to prevent vanishing gradients early in learning, as proposed in \cite{goodfellow2014generative}, we train the generator to maximize $\log(D(G(\mathbf{z}, \mathbf{c}), \mathbf{c}))$ rather than minimizing $\log(1- D(G(\mathbf{z}, \mathbf{c}), \mathbf{c}))$. Additionally, as in adversarial training, the generator is trained to fool the classifier by maximizing the cross-entropy loss for randomly selected incorrect class labels. Following \cite{mukherjee2021privgan}, we improve the convergence by initializing  classifier $Q$ with weights obtained from pre-training it on the train set using the corresponding membership indices as class labels. Finally, we train the discriminator and generator for  $K$ epochs, without training the classifier to allow for the generative model to recover a rough estimate of the data distribution.

\begin{remark}
Even though privacy preservation via PIGAN comes  at an additional computational cost of training classifier $Q$ relative to a non-private GAN, it can be trained more efficiently compared to the PrivGAN architecture proposed in \cite{mukherjee2021privgan}, which requires training multiple generator-discriminator pairs 
to convergence, and hence, requires a larger training dataset.
\end{remark}

\begin{algorithm*}[!h]
\resizebox{0.85\textwidth}{!}{
\begin{minipage}{\textwidth}
\caption{Training PIGAN using stochastic gradient descent}
 \label{alg:algorithm}
\begin{algorithmic}[1]
\STATE Partition dataset $X$ into $N$ non-overlapping subsets, index all points in each subset by $c=1,\dots,N$. 
\FOR{number of training iterations}
\STATE Sample a minibatch of $m$ data points $\{(x^{(1)}, c^{(1)}),\dots,(x^{(m)}, c^{(m)})\}$
from the data distribution $p_{\mathbf{x}}(x)$.
\STATE Sample a minibatch of $m$ noise samples $\{z^{(1)},\dots,z^{(m)}\}$ from   $p_{\mathbf{z}}(z)$, and $m$ latent samples $\{c^{(1)},\dots,c^{(m)}\}$ from $p_{\mathbf{c}}$. 
\STATE Update the discriminator by ascending its stochastic gradient:
$$\nabla_{\theta_d} \frac{1}{m} \sum_{i=1}^m \Big[\log(D(x^{(i)}, c^{(i)})] + \log(1- D(G(z^{(i)}, c^{(i)}), c^{(i)}))\Big]$$
\STATE Update the classifier by ascending its stochastic gradient:
$\nabla_{\theta_q} \frac{1}{m} \sum_{i=1}^m \log(Q(G(z^{(i)}, c^{(i)})) $
\STATE Sample a minibatch of $m$ noise samples $\{z^{(1)},\dots,z^{(m)}\}$ from   $p_{\mathbf{z}}(z)$, and $m$ latent samples $\{c^{(1)},\dots,c^{(m)}\}$ from $p_{\mathbf{c}}$. 
\STATE Update the generator by descending its stochastic gradient:
\begin{align}
\nabla_{\theta_g} \frac{1}{m} \sum_{i=1}^m \Big[& \log(1- D(G(z^{(i)}, c^{(i)}), c^{(i)})) 
+ \lambda \log(Q(G(z^{(i)}, c^{(i)})) \Big]\notag
\end{align}
\ENDFOR
\end{algorithmic}\label{alg:pigan}
\end{minipage}
 }
\end{algorithm*}

\section{Evaluation Metrics}\label{sec:evaluation metrics}
In this section, we briefly describe the metrics used in our experiments to compare different generative models in terms of the level of privacy they provide and the quality of samples they generate.
 
\subsection{Membership Attack Vulnerability}
We measure a model's privacy loss level, by empirically evaluating the success of adversarial attacks introduced in \cite{hayes2019logan,hilprecht2019monte,mukherjee2021privgan}. We assume that the adversary has a suspect dataset that contains samples from the training set and a holdout set, and that it knows the size of the training set $m$. We consider the following MIAs:
\begin{itemize} 
    \item \textbf{White-box (WB) attack:} The adversary has access to the trained discriminator model. As proposed in  \cite{hayes2019logan}, the attacker can use the discriminator outputs to rank  samples in its dataset, since the discriminator will assign higher confidence scores to train data samples if the model has overfit on the training set. The attacker predicts the $m$ samples with the highest scores as members of the training set. The WB attack was extended to the multiple discriminator setting of PrivGAN in \cite{mukherjee2021privgan}, by ranking the samples using the max (or mean) score among all discriminator scores assigned to a record. Similarly, the WB attack for PIGAN uses $\max\limits_{\mathbf{c}\in \{1,\dots,N\}} D(\mathbf{x},\mathbf{c})$ to rank record $\mathbf{x}$.
     \item \textbf{Total Variation Distance (TVD) attack:} It was shown in \cite{mukherjee2021privgan} that the total variation distance between the distribution of the discriminator scores on train and holdout sets, provides an upper limit on the accuracy of attacks that use discriminator scores. 
     In the multiple discriminator setting, the maximum TVD is used as the upper limit, and for PIGAN, we use $\max\limits_{\mathbf{c}\in \{1,\dots,N\}} \text{TVD}(P_{\mathbf{c}}, Q_{\mathbf{c}})$, where $P_{\mathbf{c}}$ and $Q_{\mathbf{c}}$ denote the distribution of $D(\mathbf{x}, \mathbf{c})$ for $\mathbf{x}$ in  train and holdout sets.

    \item \textbf{Monte-Carlo (MC) attack:} The adversary has access to the generator or a set of generated samples. \cite{hilprecht2019monte} designs two attacks: (1) single membership inference (MC-Single), where the adversary aims to identify all records that are a member of the training set, and (2) set membership inference (MC-Set), where the adversary aims to determine which of two given sample sets are a subset of the  training set. We assume the adversary dataset contains $M$ samples from the train set and $M$ samples from the holdout set. The adversary uses a metric $f(\mathbf{x})$ to rank  samples $\mathbf{x}$ in its dataset.
    \begin{itemize}
        \item[$\circ$] \textbf{MC-Single:} The $M$ samples with the highest scores are predicted to be training set members, and attack accuracy is defined as the fraction of correctly labeled samples.
        
        \item[$\circ$] \textbf{MC-Set:} The set containing most of the top $M$ samples with the highest scores is identified as the set used to train the model. The attack accuracy is defined as the average success rate of correctly identifying the training subset.
     \end{itemize}  
    
    While \cite{hilprecht2019monte} proposes various $f(\mathbf{x})$ as the ranking metric, their experiments show that the following distance-based metric using the median heuristic outperforms others.

    \begin{align}
     f(x) = \frac{1}{n}\sum\limits_{i=1}^n \mathbbm{1} \{x_{g_i} \in U_{\epsilon}(x)\}, \, U_{\epsilon}(x) = \{x'|d(x,x')\leq \epsilon\}, \notag
    \end{align}

    where  $\{x_{g_1},\dots,x_{g_n}\}$ are generated samples, and $U_{\epsilon}(x)$ denotes the $\epsilon$-neighbourhood of $x$ with respect to some distance $d(.,.)$, and $\epsilon$ is computed as 
    $$\epsilon = \text{median}_{1\leq i\leq 2M} (\min\limits_{1\leq j\leq n} d(x_i, x_{g_j})).$$ 
    For image datasets, we use the Euclidean distance on few of the top components resulting from applying the Principal Components Analysis (PCA) transformation on pixel intensities of any two images \cite{hilprecht2019monte}.

\end{itemize}

\subsection{Sample Fidelity}
 Generative models should synthesize diverse high-fidelity samples that agree with human perceptual judgments. We use the following quantitative measures to compare different image synthesis models, and refer the reader to \cite{borji2019pros} for a comprehensive review on GAN evaluation measures.
\begin{itemize}
      \item \textbf{Inception Score (IS):} A pre-trained neural network (generally Inception Net \cite{szegedy2016rethinking}) is used to assess if samples are highly classifiable and diverse with respect to class labels $\mathbf{y}$ \cite{salimans2016improved}. This metric measures the KL divergence between the  conditional class distribution $p(\mathbf{y}|\mathbf{x})$ and marginal class distribution $p(\mathbf{y})$. For a high-quality generator, $p(\mathbf{y}|\mathbf{x})$ has low entropy (highly classifiable images), while $p(\mathbf{y})$ is high-entropy (diverse images), resulting in \textit{high} IS. Despite its wide adoption as a metric, IS does not take into account the statistics of real images, and is not able to capture how well the data distribution is approximated by the generator. 
      
      \item \textbf{Fr\'echet Inception Distance (FID):} The FID score uses embeddings from the penultimate layer of Inception Net to measure the distance between real and generated image  distributions  \cite{heusel2017gans}. The embedding vector distributions are modeled as multivariate Gaussians, their mean and covariance are estimated, and their distance is measured using the Fr\'echet (Wasserstein-2) distance. A high-quality generator is able to approximate the real data distributions well resulting in \textit{low} FID. 
      
      \item {\textbf{Intra-FID Score}:} For class-conditional models, \cite{miyato2018cgans} introduced Intra-FID as a metric to assess the quality of a conditional generative distribution, by calculating the FID score separately for each conditioning and reporting the average score. 
      They empirically demonstrated that Intra-FID is able to capture visual quality, intra-class diversity and conditional consistency.

      \item \textbf{Classification performance:} In order to evaluate how well the generative model captures the data distribution, it is proposed in \cite{ravuri2019classification} to use the synthetic samples for a downstream task such as training a  classifier to predict the real data class labels. 
      The classification accuracy can be interpreted as a recall measure that relates to the diversity of generated samples  \cite{shmelkov2018good}.
      \end{itemize}

\section{Experiments}\label{sec:experiments}
We empirically investigate the effectiveness of the proposed regularization in preserving privacy on  two widely-adopted image datasets, and compare the privacy-fidelity trade-off achieved by PIGAN with respect to non-private and private baselines. We consider the following datasets: (1) \textbf{Fashion-MNIST} \cite{xiao2017fashion}  which contains $70,000$ labeled $28\times28$ grayscale images representing $10$  clothing categories, 
and (2) \textbf{CIFAR-10} \cite{krizhevsky2009learning}, which contains $60,000$ labeled $32\times32$ color images representing $10$ categories such as planes, cars and ships.

\subsection{Setup and Model Architectures}
As conventional in  MIA literature on generative models, to trigger overfitting, we select a random $10\%$ subset from each dataset for training and use the remainder for evaluating the models. Results are reported by averaging attack accuracy and fidelity metrics corresponding to $10$ experiments run with different train-test splits. In addition to standard non-private GANs, we compare PIGAN with PrivGAN \cite{mukherjee2021privgan} and DPGAN \cite{DPGAN}. We do not consider PATE-GAN \cite{jordon2018pate} as it was shown to generate low quality samples for image datasets in \cite{chen2020gs}. 

\subsubsection{Model Architectures} Following existing work in this area \cite{hayes2019logan,hilprecht2019monte,mukherjee2021privgan, DPGAN}, we adopt the deep convolutional generative adversarial network (DCGAN) \cite{radford2015unsupervised} architecture for all  models trained on both datasets, and generate class labels by implementing the conditional variant of DCGAN \cite{mirza2014conditional}. All models are trained for $300$ epochs on Fashion-MNIST and $500$ epochs on CIFAR-10, with a mini-batch size of $128$ using the Adam optimizer with learning rate $0.0002$ and $\beta_1=0.5$.  For PIGAN, we use the same architecture as the non-private GAN modified to take the membership latent code $\mathbf{c}$ as an additional input to the first layer. The regularization classifier $Q$ is implemented similar to the discriminator, but without taking class and membership labels as inputs.  PrivGAN uses the same architecture as non-private GAN for all generator-discriminator pairs and its privacy discriminator is identical to all other discriminators differing only in the activation function of the final layer\footnote{\url{https://github.com/microsoft/privGAN}}. We implement DPGAN using the same architecture and train it in a  differentially private manner\footnote{Implemented using \url{https://github.com/tensorflow/privacy}.} with $\delta=10^{-4}$ (typically chosen as the inverse of train set size). All  model architectures are described in detail in Appendix~\ref{app:architecture}.

\subsubsection{Attack Parameters} WB and TVD attacks are tested on the  $90\%$ holdout samples 
not used for training. MC attacks are conducted using $10^5$ synthetic samples generated by each model, assuming that the adversary dataset contains $M=100$ samples from each of the train and holdout sets. For each experiment, MC attacks are repeated $20$ times and average accuracy is reported. As in \cite{hilprecht2019monte}, we use the Euclidean distance on the top $40$ 
PCA components. 
From each dataset's $90\%$ holdout samples, we select a random $10\%$ subset to compute the PCA transformation and use the rest for testing  MC attacks. 
Based on the train-holdout split of the attacker's  dataset, a baseline attack that uses random guessing will result in  $10\%$ WB attack accuracy, and $50\%$ MC attack accuracy. 

\subsection{Evaluation}

\subsubsection{Privacy Preservation with PIGAN} 
We quantitatively assess how the proposed regularization in PIGAN prevents privacy loss with respect to its performance against various MIA attacks. Fig.~\ref{fig:performance} shows the attack vulnerability and generated sample quality of PIGAN trained with $N=2$, as regularization penalty $\lambda$ varies between $[0,30]$. For both datasets,  WB and MC-Set attack accuracy are reduced for stronger regularization values. With $\lambda=10$, PIGAN provides  $34.8\%$ and $13\%$ improvement for Fashion-MNIST against WB and MC-Set attacks compared to non-private GAN, respectively, which comes  only with $3.2\%$ degradation in downstream classification performance, and negligible increase in the Intra-FID score. For CIFAR-10, regularization with $\lambda=10$  reduces WB and MC-Set attack success by $21.8\%$ and $2.5\%$, respectively, with small loss in terms of Intra-FID and classification accuracy. Further quantitative evaluation of PIGAN in terms of TVD and MC-Single attacks, and inception and FID scores, as well as a qualitative comparison of the generated samples are provided in Appendix~\ref{app:lambda}.

\begin{figure*}[!htb]
\centering
\begin{subfigure}{.87\textwidth}
  \centering
  \includegraphics[width=\textwidth]{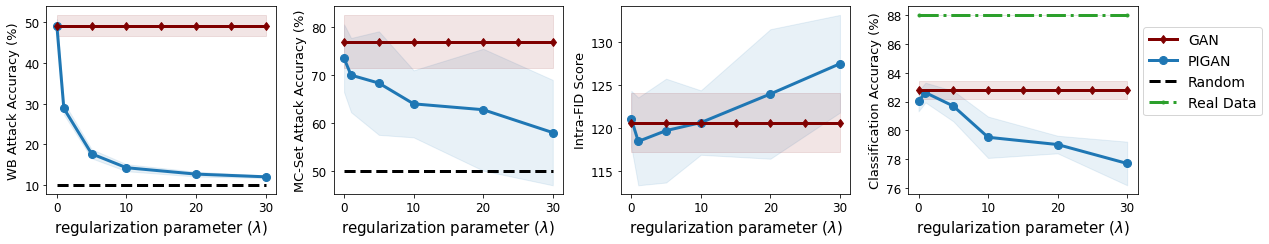} 
  \vspace{-6mm}
  \caption{Fashion-MNIST}
  \label{fig:sub-first}
\end{subfigure}
\\
\begin{subfigure}{.87\textwidth}
  \centering
  \includegraphics[width=1\textwidth]{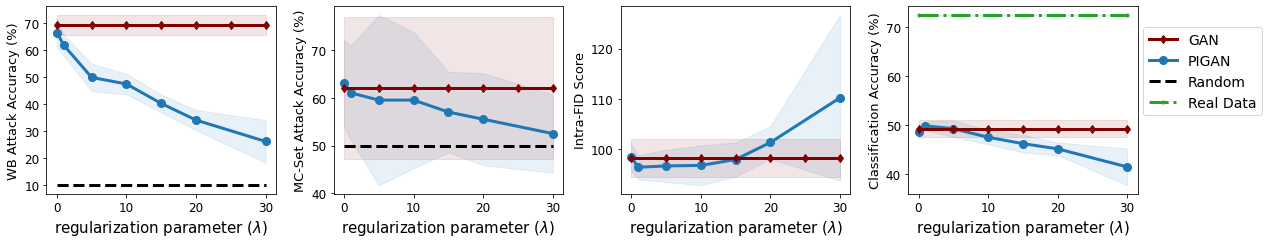}
    \vspace{-6mm}
  \caption{CIFAR 10}
  \label{fig:sub-second}
\end{subfigure}
\vspace{-2mm}
\caption{Privacy and fidelity measures for PIGAN trained with $N=2$ and  various regularization $\lambda$.}
\label{fig:performance}
\end{figure*}

 \setlength{\textfloatsep}{1pt}
 \begin{figure*}[!htb]
\centering
\begin{subfigure}{.9\textwidth}
  \centering
  \includegraphics[width=1\textwidth]{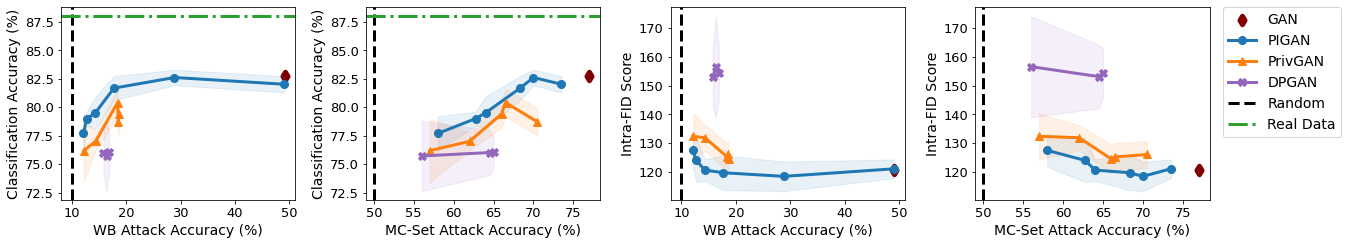} 
  \vspace{-6mm}
  \caption{Fashion-MNIST}
  \label{fig:sub-first}
\end{subfigure}
\\
\begin{subfigure}{.9\textwidth}
  \centering
  \includegraphics[width=1\textwidth]{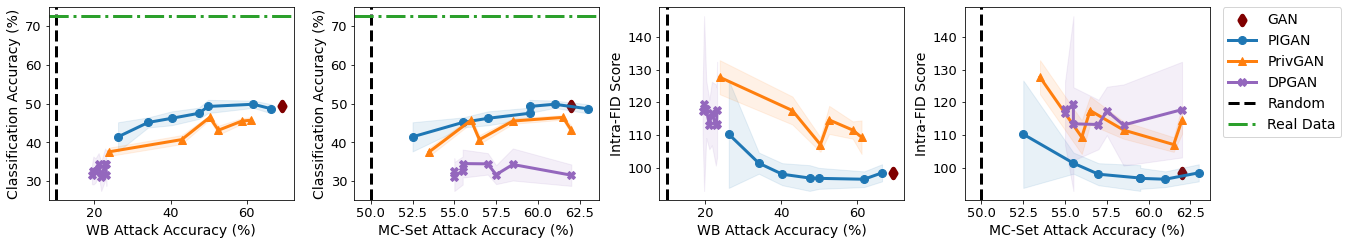}  
  \vspace{-6mm}
  \caption{CIFAR 10}
  \label{fig:sub-second}
\end{subfigure}
\vspace{-3mm}
\caption{Privacy-fidelity trade-off  achieved with different private models.}
\label{fig:tradeoff}
\vspace{-2mm}
\end{figure*}

\subsubsection{White box attacks using discriminator confidence scores}\label{app:wb attack}
As pointed out in \cite{mukherjee2021privgan}, GANs can be visually compared in terms of their resistance to WB attacks based on the distribution of confidence scores predicted by the discriminator on train and holdout samples. For a generative model with better privacy protection, the distributions will be more similar and the statistical differences between scores assigned to train and non-train (holdout) samples can not be exploited by an adversary.  As shown in Fig.~\ref{fig:dist lambda}, the distributions for non-private GAN are very different, while they they overlap more for PIGAN  as regularization parameter $\lambda$ is increased.

\begin{figure*}[!h]\centering
\begin{subfigure}{.22\textwidth}\centering
  \includegraphics[width=1\textwidth]{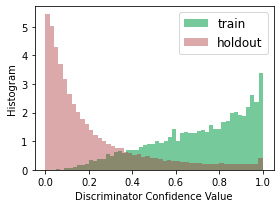} 
\end{subfigure}
\begin{subfigure}{.22\textwidth}\centering
  \includegraphics[width=1\textwidth]{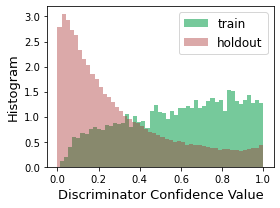}
\end{subfigure}
\begin{subfigure}{.22\textwidth}\centering
  \includegraphics[width=1\textwidth]{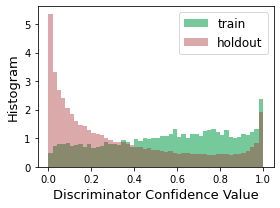} 
\end{subfigure}
\begin{subfigure}{.22\textwidth}\centering
  \includegraphics[width=1\textwidth]{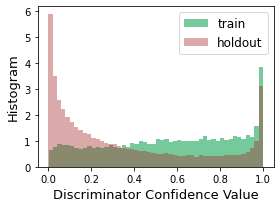}
\end{subfigure}
\\
\begin{subfigure}{.22\textwidth}\centering
  \includegraphics[width=1\textwidth]{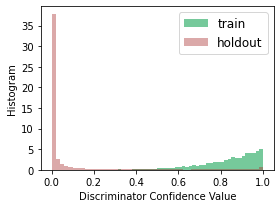}   \caption{GAN}
\end{subfigure}
\begin{subfigure}{.22\textwidth}\centering
  \includegraphics[width=1\textwidth]{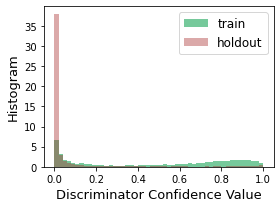}
  \caption{PIGAN ($\lambda=1$)}
\end{subfigure}
\begin{subfigure}{.22\textwidth}\centering
  \includegraphics[width=1\textwidth]{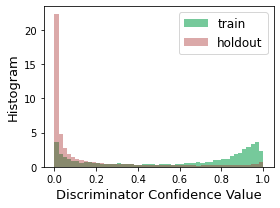} 
 \caption{PIGAN ($\lambda=10$)}
\end{subfigure}
\begin{subfigure}{.22\textwidth}\centering
  \includegraphics[width=1\textwidth]{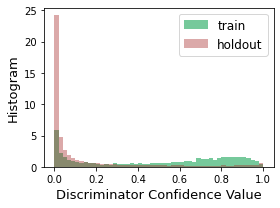}
 \caption{PIGAN ($\lambda=20$)}
\end{subfigure}
\vspace{-2mm}
\caption{Distribution of PIGAN's discriminator confidence score on train and holdout data compared to non-private GAN for Fashion-MNIST (top) and CIFAR-10 (bottom).}
\label{fig:dist lambda}
\end{figure*}

\begin{table*}[!htbp]
\centering
\ra{1.3}
\resizebox{0.8\textwidth}{!}{%
\begin{tabular}{@{}rrrccccc@{}}\toprule
& \multicolumn{3}{c}{Fashion-MNIST} & \phantom{abc}& \multicolumn{3}{c}{CIFAR-10}  \\
\cmidrule{2-4} \cmidrule{6-8}  
& PIGAN  & PrivGAN & DPGAN && PIGAN & PrivGAN & DPGAN \\ \midrule
WB attack ($\%$)  $\downarrow$ 
&\text{17.7$\pm$1.1\;\;} & \text{18.3$\pm$1.3\;\;}  & 16.9$\pm$0.9
&& 26.2$\pm$7.8\,& 23.8$\pm$6.8 \;& 23.0$\pm$5.4 \;\\
MC-Set attack ($\%$) $\downarrow$ 
&68.3$\pm$10.9 &\text{67.0$\pm$12.5} &\textbf{65.0$\pm$9.2}
&&52.5$\pm$8.2\;& 53.5$\pm$12.6& \textbf{55.5$\pm$9.12}\\
Inception Score $\uparrow$  
&\textbf{2.43$\pm$0.03}   & \text{2.40$\pm$0.03}   & \;\;2.19$\pm$0.08
&& \textbf{\,3.92$\pm$0.16} & 3.62$\pm$0.22& 3.49$\pm$0.15\\
FID Score $\downarrow$  
&\textbf{22.8$\pm$1.57}   & \text{33.9$\pm$1.3\;\;}   &46.1$\pm$7.4
&& \textbf{70.44$\pm$15.3} & 87.6$\pm$7.8\;\;& 75.4$\pm$9.2\;\;\\
Intra-FID Score $\downarrow$  
&\textbf{119.7$\pm$6.04}   & \text{125.2$\pm$4.8\;\;}   &154.4$\pm$8.6\;\;
&& \textbf{110.2$\pm$16.4}& 127.4$\pm$5.2\;\;\;\;& 113.4$\pm$11.2\;\;\\
 Classification Accuracy ($\%$) $\uparrow$
 &\textbf{81.7$\pm$1.1\;}   & \text{79.6$\pm$1.7\;\;}   &76.1$\pm$1.4
 && \textbf{41.4$\pm$3.7}& 37.6$\pm$1.1\;\;& 34.6$\pm$3.6\;\;\\
\bottomrule
\end{tabular}
}
\vspace{-2mm}
\caption{Comparison of privacy-preserving models trained to achieve similar WB attack accuracy levels.}
\label{tb:comparison}

\bigskip
\medskip

\ra{1.}
\resizebox{0.82\textwidth}{!}{%
\begin{tabular}{@{}rlcllccccc@{}}\toprule
& \multicolumn{4}{c}{Privacy Attacks} & \phantom{}& \multicolumn{4}{c}{Fidelity Scores}  \\
\cmidrule{2-5} \cmidrule{7-10}
& WB  $\downarrow$ & TVD $\downarrow$ & MC-Set  $\downarrow$ & MC-Single $\downarrow$ && Inception  $\uparrow$ & FID  $\downarrow$ & Intra-FID  $\downarrow$ & Classification $\uparrow$ \\ \midrule
$N=2$ 
&28.8$\pm$1.5 &0.448$\pm$0.015 &70.0$\pm$7.7 &51.8$\pm$0.8
&& 2.44$\pm$0.03 & \textbf{20.7$\pm$1.2} & \textbf{118.5$\pm$5.1} &  \textbf{82.6$\pm$0.6}\\
$N=3$ 
&\textbf{18.5$\pm$1.9} &\textbf{0.422$\pm$0.016}&\textbf{63.9$\pm$10.6} & \textbf{51.4$\pm$0.6}
&& 2.43$\pm$0.04& 24.2$\pm$1.8& 120.1$\pm$3.2& 81.5$\pm$0.1\\
$N=4$ 
&22.7$\pm$2.6  &0.461$\pm$0.024   &65.5$\pm$11.5&51.8$\pm$0.9
&& \textbf{2.45$\pm$0.06}& 24.4$\pm$1.3& 120.7$\pm$4.3& 81.9$\pm$0.5\\
$N=5$ 
 &21.3$\pm$2.7   &0.455$\pm$0.025 &64.5$\pm$9.3&51.6$\pm$0.6
 && 2.34$\pm$0.04& 22.9$\pm$1.7& 125.1$\pm$4.8& 82.2$\pm$0.7 \\
 $N=6$ 
 &22.3$\pm$2.8   &0.458$\pm$0.027 &69.5$\pm$3.5&51.9$\pm$1.1
 && \textbf{2.44$\pm$0.03}&25.7$\pm$2.1 & 119.4$\pm$4.6& 81.5$\pm$0.8\\
\bottomrule
\end{tabular}
}
 \vspace{-2mm}
\caption{Privacy and fidelity measures of PIGAN trained with $\lambda=1$ and different $N$ on Fashion-MNIST.}
\label{tb:number N}
\end{table*}

\subsubsection{Privacy-Fidelity Trade-off}
To simultaneously compare the privacy loss and generated sample quality achieved by PIGAN ($N=2$) with the baselines (GAN, PrivGAN ($N=2$) and DPGAN), we use the trade-off curves presented in Fig.~\ref{fig:tradeoff}. The curves are generated by training PIGAN and PrivGAN for a range of $\lambda$ and by training DPGAN for a range of $\epsilon$ obtained from different clipping and noise levels. 
For Fashion-MNIST and CIFAR-10, the two left curves in Figs.~\ref{fig:tradeoff} (a) and (b)  are strictly higher (other than one point for CIFAR-10) than the curves of all other private methods, and the two right curves  are strictly lower. 
Therefore, for a specified level of privacy (i.e., a fixed point on the $x$-axis), a   classifier  trained on data that is generated by PIGAN will generalize better compared to one trained on data generated by other private models. Moreover, based on the lower Intra-FID scores shown in two right curves in Figs.~\ref{fig:tradeoff} (a) and (b), PIGAN is able to generate data with better sample quality and higher intra-class diversity. As expected for differentially private GANs, DPGAN is inferior to  PIGAN and PrivGAN for both datasets, and provides privacy at the cost of losing on downstream utility and generated sample quality. The models are compared 
visually based on generated sample quality in Appendix~\ref{app:tradeoff}.

Table~\ref{tb:comparison} reports the different fidelity metrics for each model based on a fixed WB accuracy point on the curves in Fig.~\ref{fig:tradeoff}, i.e., when the models are trained (with proper parameters $\lambda$, $\epsilon$) to achieve (almost) equal WB attack accuracy. Specifically, considering WB attack accuracy close to $17\%$ for Fashion-MNIST, PIGAN ($\lambda=5$) outperforms PrivGAN ($\lambda=0.01$) and DPGAN by $2.1\%$ and $5.6\%$ in terms of downstream classification, and by $5.5$ and $34.7$ in terms of Intra-FID score, respectively. For CIFAR-10, when considering WB attack success of around $23$-$26\%$, PIGAN ($\lambda=30$) generates data with $3.8\%$ and $6.8\%$ higher classification utility compared to  PrivGAN ($\lambda=20$) and DPGAN, respectively.

\subsubsection{Hyperparameter ${N}$} 
In theory, increasing $N$ should improve privacy since the classifier $Q$ would prevent the generative model from overfitting to specific subsets of the data. However, larger $N$ results in less training samples for each membership group, which impacts the learned generative distribution for each group and the resulting sample quality.  Privacy and fidelity measures for PIGAN trained on Fashion-MNIST for $\lambda=1$ are reported in Table~\ref{tb:number N} as $N$, the cardinality of latent code $\mathbf{c}$, is changed from $2$ to $6$. An equal train set size is used for training PIGAN with different $N$. Increasing $N$ from $2$ to $3$ reduces the success of WB and MC-Set attacks by $10.3\%$ and $6.1\%$, respectively. This improvement in privacy only degrades the Intra-FID score by $1.6$ points and the classification performance by $1.1\%$. Increasing $N$ beyond $3$ does not seem to improve the privacy-utility trade-off on Fashion-MNIST considerably.
 Parameters $N$ and $\lambda$ are hyper-parameters that can be tuned for the downstream task, and will likely depend on the dataset as mentioned by (Mukherjee et al. 2021) based on experiments for PrivGAN. 

\subsubsection{Training Efficiency of PIGAN} \label{subsec:efficiency}
We empirically demonstrate that the proposed PIGAN model can be trained more efficiently  compared to PrivGAN, which requires training multiple generator-discriminator pairs.
The following table summarizes the number of trainable parameters in each model. Compared to the GAN architecture, PIGAN has $1.33\times$ more parameters when trained on Fashion-MNIST, while PrivGAN has $2.28\times$  more parameters. 

\vspace{0.2cm}
\begin{table}[t]
  \centering
 \resizebox{0.7\columnwidth}{!}{
\begin{tabular}{@{}rccc@{}}\toprule
\phantom{a} & \multicolumn{3}{c}{$\#$parameters ($\times10^6$)}  \\ \cmidrule{2-4}
& GAN & PIGAN  & PrivGAN     \\ \midrule
Fashion-MNIST  
&2.24 & 2.98  & 5.11\\
CIFAR-10
&  1.93 &  2.59 & 4.39\\
\bottomrule
\end{tabular}
}
\end{table}
\vspace{0.2cm}

As opposed to PIGAN, the number of parameters in PrivGAN increases (almost linearly) with $N$, and each generator-discriminator pair in PrivGAN is trained with a smaller fraction of the data. The training procedure for PrivGAN is such that each  generator-discriminator pair is only trained on $1/N^{\text{th}}$ of the training data. Therefore, for equal batch size and number of epochs, PrivGAN parameters are updated less compared to PIGAN. We conduct the following two experiments to compare the models in terms of training efficiency, and  report the results in Table \ref{tb:cost-cases}. In both experiments, $N=2$ and $\lambda$ is chosen such that PIGAN  generates samples with relatively similar (or slightly stronger\footnote{Note that in cases that the WB attack accuracy of PIGAN is less than PrivGAN, PIGAN provides stronger privacy protection.}) privacy levels compared to PrivGAN in terms of WB attack accuracy. Note that for this choice of $\lambda$, PrivGAN has lower vulnerability to MC-Set attacks  compared to PIGAN. %

\paragraph{\textbf{Experiment 1:}} We trained PIGAN for less number of epochs compared to PrivGAN. For Fashion-MNIST, PIGAN was trained for $200$ epochs (with $K=100$), while PrivGAN was trained for 300 epochs. For CIFAR-10 we used $400$ epochs (with $K=300$) for PIGAN compared to $500$ epochs for PrivGAN.  As observed from Table \ref{tb:cost-cases}(a), when the models are trained to achieve similar levels of privacy in terms of WB attacks, PIGAN outperforms PrivGAN in terms of generated sample fidelity  despite being trained for less epochs.
\paragraph{\textbf{Experiment 2:}} We trained PIGAN on $3/4^{\text{th}}$ of the train set on both Fashion-MNIST and CIFAR-10 datasets, while we trained PrivGAN on the entire train set. From Table~\ref{tb:cost-cases}(b), it can be observed that even with $25\%$ less training samples, PIGAN is able to generate images with higher (or similar) quality in terms of classification accuracy and Intra-FID score, while providing  WB privacy levels comparable to PrivGAN for both datasets. 

\begin{table*}[!h]
\centering

\begin{subtable}[t]{\linewidth}
  \centering
  
\resizebox{0.6\textwidth}{!}{%
\begin{tabular}{@{}rrcclr@{}}\toprule
& \multicolumn{2}{c}{Fashion-MNIST} & \phantom{}& \multicolumn{2}{c}{CIFAR-10}  \\
\cmidrule{2-3} \cmidrule{5-6}
& PIGAN & PrivGAN && PIGAN & PrivGAN \\ 
\midrule
 WB ($\%$) $\downarrow$ 
 & \textbf{13.9$\pm$0.7} &  \text{14.3$\pm$0.8} 
 && \textbf{30.9$\pm$3.9} & \text{42.9$\pm$6.8}\\
 MC-Set ($\%$) $\downarrow$
 & { 64.5$\pm$0.7} & \textbf{\;\, 62.0$\pm$11.7}
 && 59.5$\pm$10.6 & \textbf{56.5$\pm$8.4} \\
 Intra-FID $\downarrow$
 & \textbf{119.5$\pm$4.6} & 131.9$\pm$4.4
 && \textbf{96.9$\pm$3.3 } & 117.5$\pm$4.3 \\
 Classification Accuracy ($\%$) $\uparrow$
 & \textbf{80.2$\pm$0.8}  &\text{77.0$\pm$0.8}
 && \textbf{45.3$\pm$1.4} & 40.7$\pm$1.2\\
\bottomrule
\end{tabular}
}
\caption{Experiment 1}
\label{tb:exp1}
\end{subtable}

\vspace*{2mm}
\centering

\begin{subtable}[t]{\linewidth}
  \centering
  
\resizebox{0.6\textwidth}{!}{%
\begin{tabular}{@{}rrrcrr@{}}\toprule
& \multicolumn{2}{c}{Fashion-MNIST} & \phantom{}& \multicolumn{2}{c}{CIFAR-10}  \\
\cmidrule{2-3} \cmidrule{5-6}
& PIGAN & PrivGAN && PIGAN & PrivGAN \\ 
\midrule
 WB ($\%$) $\downarrow$ 
 & \textbf{12.0$\pm$0.3 } &  12.2$\pm$0.5 
 && \textbf{39.8$\pm$2.7 } & 42.9$\pm$6.8\\
 MC-Set ($\%$) $\downarrow$
 & \text{\;62.1$\pm$13.1} & \textbf{57.0$\pm$8.7}
 && \text{\;59.0$\pm$11.5} & \textbf{56.5$\pm$8.4} \\
 Intra-FID $\downarrow$
 & \textbf{122.2$\pm$7.0 } & 132.4$\pm$7.9
 && \textbf{100.1$\pm$3.1 } & 117.5$\pm$4.3 \\
 Classification Accuracy ($\%$) $\uparrow$
 & \textbf{77.1$\pm$0.9 }  &76.2$\pm$2.8
 && \textbf{43.0$\pm$1.7 } & 40.7$\pm$1.2\\
\bottomrule
\end{tabular}
}
\caption{Experiment 2}
\label{tb:exp2}
\end{subtable}
 
 \caption{PIGAN vs PrivGAN in terms of training efficiency.}
\label{tb:cost-cases}
\end{table*}

 \subsubsection{Privacy Range}
As observed from Fig.~\ref{fig:tradeoff}, PIGAN provides a wider privacy range against WB attacks compared to PrivGAN (as also noted in \cite{mukherjee2021privgan}) and DPGAN, especially for Fashion-MNIST. It is worth noting that PrivGAN requires training multiple G-D pairs with a more trainable parameters compared to PIGAN, and therefore, it requires more training data and updates compared to PIGAN, especially for large values of $N$. PrivGAN's smaller privacy range, may potentially be due to the additional computational cost, since all $N$ G-D pairs may not converge as fast as models with a single generator and discriminator (e.g., GAN and PIGAN) when trained for the same number of iterations. Consequently, it may be less susceptible to privacy attacks due to less fitting to the data. Based on the results reported in Fig.~\ref{fig:tradeoff}, without regularization ($\lambda= 0$), PrivGAN is $30.7\%$ and $31.4\%$ less susceptible to WB attacks compared to non-private GAN for Fashion-MNIST and CIFAR-10, respectively. However, as expected, with no privacy regularization using $\lambda=0$, PIGAN achieves almost the same WB attack accuracy as non-private GAN. 

\begin{figure}[h!]
  \centering
\begin{subfigure}{0.24\textwidth}
  \centering
  \includegraphics[width=\textwidth]{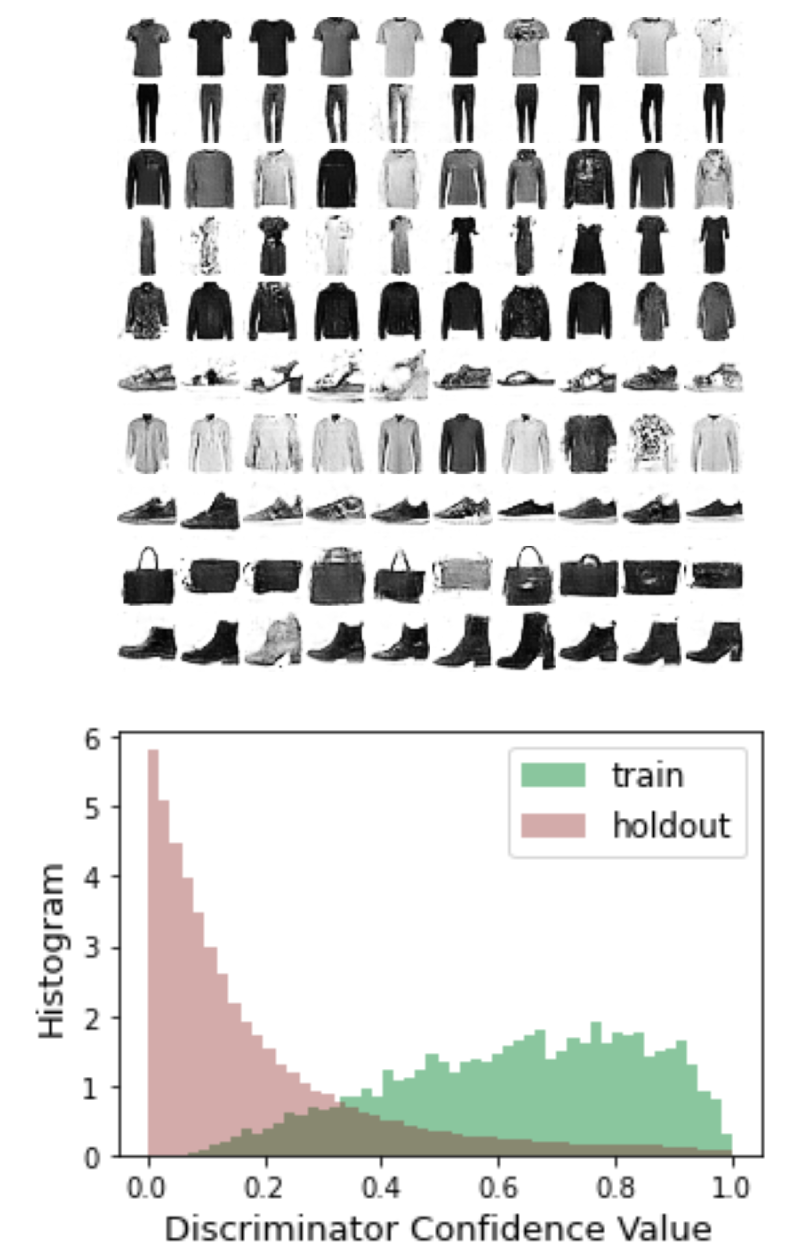}
  \caption{$\mathbf{c}=1$}
\end{subfigure}
\begin{subfigure}{0.22\textwidth}
  \centering
  \includegraphics[width=\textwidth]{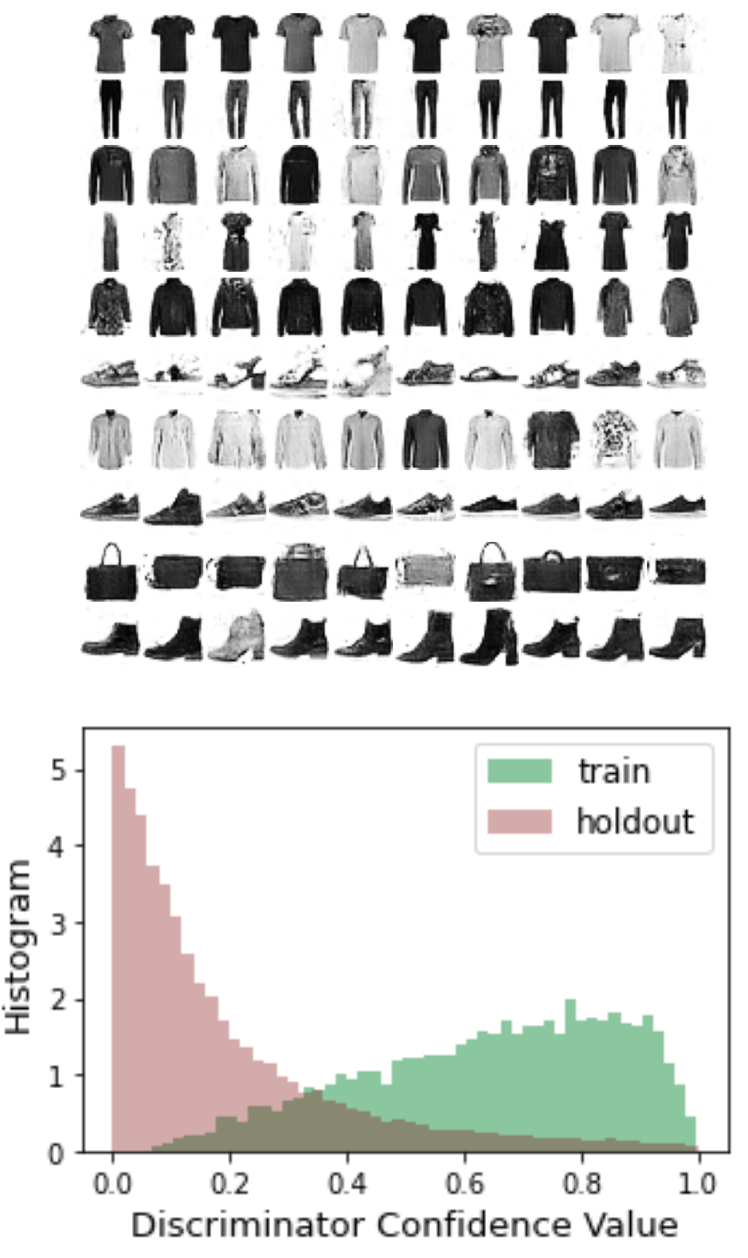}  
  \caption{$\mathbf{c}=2$}
\end{subfigure}
\vspace{-3mm}
\caption{\small PIGAN trained on Fashion-MNIST with $N=2$, $\lambda=0$.}
\label{fig:lambda0}
\end{figure}

\subsubsection{No Penalty $\boldsymbol{\lambda=0}$} 
To further demonstrate the vulnerability of GANs to privacy attacks, we train PIGAN on Fashion-MNIST with $N=2$ and $\lambda=0$, such that the generator is conditionally trained on different subsets of the data but is not penalized for memorizing the data in each subset. As shown in Fig.~\ref{fig:lambda0}, synthetic samples generated by PIGAN are visually indistinguishable for $\mathbf{c}=1$ and $\mathbf{c}=2$, i.e., when a GAN is trained on different subsets of the dataset. However, there is a clear distinction between the discriminator's confidence about the realness of a sample from the train or holdout set for both values of $\mathbf{c}$. Such statistical differences arise from overfitting to the training set and can be exploited by an adversary. By imposing regularization with $\lambda=1$, the WB and MC-Set attack accuracy diminish from $49.1\%$ to $28.8\%$, and from $73.5\%$ to $70\%$, respectively (as reported in Table~\ref{tb:number N}). PIGAN's discriminator confidence score distributions   are compared for various $\lambda$ in  Appendix~\ref{app:wb attack}, which shows that the distributions become more similar as $\lambda$ is increased, meaning the discriminator assigns closer confidence values to real samples coming from train and holdout sets.

\section{Conclusions}
We propose a membership privacy-preserving training framework for deep generative models that mitigates overfitting to training data and leakage of sensitive information. Our proposed model, referred to as PIGAN, learns to estimate almost identical generative distributions for data with different membership, using an information-theoretic regularization term that aims to minimize the divergence between the distributions. We show that this can be implemented using a multi-class classifier at relatively low cost compared to private GANs that train multiple generators and discriminators. Our experiments demonstrate the resilience of PIGAN against several well-known MIA attacks without considerable degradation in generated data  fidelity, and show that PIGAN outperforms alternative private GANs   in terms of various utility measures   achieving improved privacy-fidelity trade-offs. As part of our future work we plan to implement PIGAN on tabular datasets with sensitive attributes such as health-related data. Other future work includes investigating achievable privacy levels when the information-theoretic regularization is imposed on discriminator confidence values rather than generated images, and comparing the mitigation of the two methods against various privacy attack types.


\paragraph{Disclaimer}
This paper was prepared for informational purposes by
the Artificial Intelligence Research group of JPMorgan Chase \& Co. and its affiliates (``JP Morgan''),
and is not a product of the Research Department of JP Morgan.
JP Morgan makes no representation and warranty whatsoever and disclaims all liability,
for the completeness, accuracy or reliability of the information contained herein.
This document is not intended as investment research or investment advice, or a recommendation,
offer or solicitation for the purchase or sale of any security, financial instrument, financial product or service,
or to be used in any way for evaluating the merits of participating in any transaction,
and shall not constitute a solicitation under any jurisdiction or to any person, if such solicitation under such jurisdiction or to such person would be unlawful.

\bibliographystyle{icml2022}
\bibliography{references}

\newpage

\appendix
\onecolumn
\section*{\centering\huge{Appendix}}
\vspace{0.2cm}
\input{SupplementaryMaterial}

\end{document}

%% file: SupplementaryMaterial.tex
 


\title{Appendix}

\section{Experimental Settings}\label{app:experiments}
All experiments are implemented in Keras and run with a single NVIDIA T4 GPU on an Amazon Web Services g4dn.4xlarge instance.

\subsection{Model Architectures}\label{app:architecture}
The proposed privacy preservation framework is independent of the architecture and should generalize to alternative models, particularly more complex models that generate higher fidelity
samples since they would be more susceptible to MIAs \cite{hayes2019logan}.
However, since the focus of our work is privacy-preserving
techniques, we follow existing work in this area \cite{hayes2019logan,hilprecht2019monte,mukherjee2021privgan, DPGAN} and adopt the deep convolutional generative adversarial network (DCGAN) \cite{radford2015unsupervised} as our base architecture and we compare various privacy-preserving techniques for the same architecture. Models are implemented in their conditional format to provide control over the generated class labels. All models are trained for $300$ epochs on Fashion-MNIST and $500$ epochs on CIFAR-10, and we use the Adam  optimizer \cite{kingma2014adam} with learning rate $\alpha=0.0002$ and momentum $\beta_1=0.5$, and a batch size of $128$.   The generator and discriminator architectures for non-private GAN are presented in Table~\ref{tb:gan models}. We use BN to denote batch normalization with momentum $0.9$, and LReLU to denote Leaky Rectified Unit with slope $\alpha=0.2$. The noise vector $\mathbf{z}$ is generated from a normal distribution with dimension $n_z=100$ for both Fashion-MNIST and CIFAR-10.  Table~\ref{tb:pigan models} provides the architectures used for PIGAN, which have been modified compared to their non-private counterparts to take the latent code $\mathbf{c}$ as an additional input, which is concatenated with other inputs in the early layers. Classifier $Q(\mathbf{x})$ is implemented using a very similar architecture as the non-private discriminator with the difference that it does not take the class labels $\mathbf{y}$ as input and that a Softmax activation is used in the final layer instead of a Sigmoid. For PrivGAN\footnote{\url{https://github.com/microsoft/privGAN}} \cite{mukherjee2021privgan}, all generator-discriminator pairs use the same architecture as non-private GAN, and the privacy discriminator is implemented identical to the classifier of PIGAN. 
DPGAN is implemented with the same architecture as non-private GAN and trained with differential privacy\footnote{Implemented using \url{https://github.com/tensorflow/privacy}} for $\delta=1e^{-4}$. For CIFAR-10, we use one-sided label smoothing for the discriminators by using a target label of $0.9$ rather than $1$.

\subsection{Other Design Choices and Hyperparameters}\label{app:hyperparam}
For both PIGAN and PrivGAN, we initialize classifier $Q$ and the privacy discriminator with weights obtained from pre-training them for $50$ epochs on the training data using membership indices as class labels. Additionally, the discriminators and generators of PrivGAN are trained for $K=100$ epochs without training the privacy discriminator \cite{mukherjee2021privgan}, while for PIGAN, they are trained for $K=200$ epochs on Fashion-MNIST and $K=400$ epochs on CIFAR-10, without updating the classifier.  

When evaluating the fidelity of data generated by the models, we use Tensorflow implementations of the Inception score and Fr\'echet Inception Distance, provided in {\url{https://github.com/tsc2017/Inception-Score}} and {\url{https://github.com/tsc2017/Frechet-Inception-Distance}}. We use $10^4$ generated samples for evaluations on Fashion-MNIST  and $2\times10^4$ generated samples for CIFAR-10. The classifier architectures used for evaluating the models based on downstream task performance is summarized in Table~\ref{tb:downstream}. MaxPool denotes a maxpool layer with pooling size of $2\times2$. The classifiers are trained with the Adam  optimizer (learning rate $\alpha=0.0002$ and momentum $\beta_1=0.5$), for $50$ epochs and a batch size of $64$ for both Fashion-MNIST and CIFAR-10. 

\begin{table*}[!h]
\centering
\ra{1.3}
\resizebox{0.85\textwidth}{!}{%
\begin{subtable}{1\linewidth}\centering
{\begin{tabular}{@{}rccc@{}}\toprule
& \multicolumn{1}{c}{Generator $G(\mathbf{z},\mathbf{y})$} & \phantom{abc}& \multicolumn{1}{c}{Discriminator $D(\mathbf{x},\mathbf{y})$}   \\ \midrule
Inputs &  $\mathbf{z}\sim \mathcal{N}(0,\mathbf{I})$ and $\mathbf{y}\in \{1,\dots,10\}$ 
&& $\mathbf{x}\sim \mathbb{R}^{28\times28\times 1}$ and $\mathbf{y}\in \{1,\dots,10\}$   
\vspace{0.15cm}
\\
Layers  
& Concatenate $\mathbf{z}$ and $\mathbf{y}$
&&  Dense $28\times28$, Reshape ($\mathbf{y}$ $\rightarrow$ $\mathbf{\hat{y}}$)\\
& Dense $7\times7\times128$, BN, LReLU, Reshape
&& Concatenate $\mathbf{x}$ and $\mathbf{\hat{y}}$ \\
&   $5\times5$ stride=$2$ Deconv $128$, BN, LReLU
&& $5\times5$ stride=$2$ Conv $64$, LReLU\\
& $5\times5$ stride=$2$ Deconv $128$, BN, LReLU
&& $5\times5$ stride=$2$ Conv 128, LReLU\\
& $3\times3$ stride=$1$ Deconv $64$, BN, LReLU
&& $5\times5$ stride=$2$ Conv 128, LReLU, Flatten\\
& $3\times3$ stride=$1$ Conv $1$, Tanh
&& Dense $1$, Sigmoid\\
\bottomrule
\end{tabular}}
\caption{Fashion-MNIST}
\label{tb:fmnist gan}
\end{subtable}%
}

\bigskip

\resizebox{0.85\textwidth}{!}{%
\begin{subtable}{1\linewidth}\centering
{\begin{tabular}{@{}rccc@{}}\toprule
& \multicolumn{1}{c}{Generator $G(\mathbf{z},\mathbf{y})$} & \phantom{abc}& \multicolumn{1}{c}{Discriminator $D(\mathbf{x},\mathbf{y})$}   \\ \midrule
Inputs &  $\mathbf{z}\sim \mathcal{N}(0,\mathbf{I})$ and $\mathbf{y}\in \{1,\dots,10\}$ 
&& $\mathbf{x}\sim \mathbb{R}^{32\times32\times3}$ and $\mathbf{y}\in \{1,\dots,10\}$   
\vspace{0.15cm}
\\
Layers  
& Concatenate $\mathbf{z}$ and $\mathbf{y}$
&&  $3\times3$ stride=$1$ Conv $64$, LReLU ($\mathbf{x}$ $\rightarrow$ $\mathbf{\hat{x}}$)\\
& Dense $4\times4\times256$, LReLU, Reshape
&& Dense $32\times32\times3$, Reshape ($\mathbf{y}$ $\rightarrow$ $\mathbf{\hat{y}}$)\\
&   $4\times4$ stride=$2$ Deconv $128$, LReLU
&& Concatenate $\mathbf{\hat{x}}$ and $\mathbf{\hat{y}}$ \\
& $4\times4$ stride=$2$ Deconv $128$, LReLU
&& $3\times3$ stride=$2$ Conv $128$, LReLU\\
& $4\times4$ stride=$2$ Deconv $64$, LReLU
&& $3\times3$ stride=$2$ Conv $128$, LReLU\\
& $3\times3$ stride=$1$ Conv $3$, Tanh
&& $3\times3$ stride=$2$ Conv $256$, LReLU, Flatten\\
&
&& Dense $1$, Sigmoid\\
\bottomrule
\end{tabular}}
\caption{CIFAR-10}\label{tab:cifar10 gan}
\end{subtable}
}
\caption{Non-private GAN generator and discriminator architectures.}
\label{tb:gan models}
\end{table*}

\begin{table*}[!h]
\centering
\hspace{1cm}\resizebox{0.8\textwidth}{!}{%
\begin{subtable}{1\textwidth}\centering
\makebox[\textwidth]{
\begin{tabular}{@{}rccccc@{}}\toprule
& \multicolumn{1}{c}{Generator $G(\mathbf{z},\mathbf{c},\mathbf{y})$} & \phantom{}& \multicolumn{1}{c}{Discriminator $D(\mathbf{x},\mathbf{c},\mathbf{y})$} & \phantom{}& \multicolumn{1}{c}{Classifier $Q(\mathbf{x})$}   \\ 
\midrule
Inputs 
&  $\mathbf{z}\sim \mathcal{N}(0,\mathbf{I})$,  $\mathbf{y}\in \{1,\dots,10\}$ 
&& $\mathbf{x}\sim \mathbb{R}^{28\times28\times 1}$, $\mathbf{y}\in \{1,\dots,10\}$  
&& $\mathbf{x}\sim \mathbb{R}^{28\times28\times 1}$  
\\
& and $\mathbf{c}\sim \text{Uniform}\{1,N\}$ 
&& and $\mathbf{c}\sim \text{Uniform}\{1,N\}$ 
&&
\vspace{0.15cm}
\\
Layers  
& Concatenate $\mathbf{z}$ and $\mathbf{y}$ ($\mathbf{z}$, $\mathbf{y}$ $\rightarrow$ $\mathbf{z}\mathbf{y}$)
&&  Dense $28\times28\times1$, Reshape ($\mathbf{y}$ $\rightarrow$ $\mathbf{\hat{y}}$)
&&  $5\times5$ stride=$2$ Conv $64$, LReLU  \\
& Dense $7\times7\times128$ ($\mathbf{z}\mathbf{y}$ $\rightarrow$ $\widehat{\mathbf{z}\mathbf{y}}$)
&& Dense $28\times28\times1$, Reshape ($\mathbf{c}$ $\rightarrow$ $\mathbf{\hat{c}}$)  
&& $5\times5$ stride=$2$ Conv 128, LReLU \\
& Dense $7\times7\times32$ ($\mathbf{c}$ $\rightarrow$ $\mathbf{\hat{c}}$)
&& Concatenate $\mathbf{x}$, $\mathbf{\hat{c}}$ and $\mathbf{\hat{y}}$ 
&& $5\times5$ stride=$2$ Conv 128, LReLU, Flatten \\
& Concatenate $\widehat{\mathbf{z}\mathbf{y}}$ and $\mathbf{\hat{c}}$, BN, LReLU, Reshape
&& $5\times5$ stride=$2$ Conv $64$, LReLU
&& Dense $N$, Softmax \\
&  $5\times5$ stride=$2$ Deconv $128$, BN, LReLU
&& $5\times5$ stride=$2$ Conv 128, LReLU
&& \\
& $5\times5$ stride=$2$ Deconv $128$, BN, LReLU
&& $5\times5$ stride=$2$ Conv 128, LReLU, Flatten
&& \\
& $3\times3$ stride=$1$ Deconv $64$, BN, LReLU
&&  Dense $1$, Sigmoid
&& \\
& $3\times3$ stride=$1$ Conv $1$, Tanh
&&
&& \\
\bottomrule
\end{tabular}}
\caption{\large Fashion-MNIST}
\label{tb:fmnist gan}
\end{subtable}%
}

\bigskip

\hspace{1cm}\resizebox{0.8\textwidth}{!}{%
\begin{subtable}{1\textwidth}\centering
\makebox[\textwidth]{
\centering
\begin{tabular}{@{}rccccc@{}}\toprule
& \multicolumn{1}{c}{Generator $G(\mathbf{z},\mathbf{c},\mathbf{y})$} & \phantom{}& \multicolumn{1}{c}{Discriminator $D(\mathbf{x},\mathbf{c},\mathbf{y})$}   & \phantom{}& \multicolumn{1}{c}{Classifier $Q(\mathbf{x})$}   \\  \midrule
Inputs 
&  $\mathbf{z}\sim \mathcal{N}(0,\mathbf{I})$,  $\mathbf{y}\in \{1,\dots,10\}$ 
&& $\mathbf{x}\sim \mathbb{R}^{32\times32\times 3}$, $\mathbf{y}\in \{1,\dots,10\}$  
&& $\mathbf{x}\sim \mathbb{R}^{32\times32\times 3}$  
\\
& and $\mathbf{c}\sim \text{Uniform}\{1,N\}$ 
&& and $\mathbf{c}\sim \text{Uniform}\{1,N\}$ 
&&
\vspace{0.15cm}
\\
Layers  
& Concatenate $\mathbf{z}$ and $\mathbf{y}$ ($\mathbf{z}$, $\mathbf{y}$ $\rightarrow$ $\mathbf{z}\mathbf{y}$)
&&  $3\times3$ stride=$1$ Conv $64$, LReLU ($\mathbf{x}$ $\rightarrow$ $\mathbf{\hat{x}}$)
&& $3\times3$ stride=$1$ Conv $64$, LReLU \\
& Dense $4\times4\times256$, Reshape ($\mathbf{z}\mathbf{y}$ $\rightarrow$ $\widehat{\mathbf{z}\mathbf{y}}$)
&& Dense $32\times32\times3$, Reshape ($\mathbf{y}$ $\rightarrow$ $\mathbf{\hat{y}}$)
&& $3\times3$ stride=$2$ Conv $128$, LReLU\\
& Dense $4\times4\times64$, Reshape ($\mathbf{c}$ $\rightarrow$ $\mathbf{\hat{c}}$)
&&  Dense $32\times32\times1$, Reshape ($\mathbf{c}$ $\rightarrow$ $\mathbf{\hat{c}}$)
&& $3\times3$ stride=$2$ Conv $128$, LReLU\\
& Concatenate $\widehat{\mathbf{z}\mathbf{y}}$ and $\mathbf{\hat{c}}$, LReLU
&& Concatenate $\mathbf{x}$, $\mathbf{\hat{c}}$ and $\mathbf{\hat{y}}$  
&& $3\times3$ stride=$2$ Conv $256$, LReLU, Flatten\\
&   $4\times4$ stride=$2$ Deconv $128$, LReLU
&&  $3\times3$ stride=$2$ Conv $128$, LReLU
&&  Dense $N$, Softmax \\
& $4\times4$ stride=$2$ Deconv $128$, LReLU
&&  $3\times3$ stride=$2$ Conv $128$, LReLU
&&\\
& $4\times4$ stride=$2$ Deconv $64$, LReLU
&& $3\times3$ stride=$2$ Conv $256$, LReLU, Flatten
&&\\
& $3\times3$ stride=$1$ Conv $3$, Tanh
&& Dense $1$, Sigmoid
&&\\
\bottomrule
\end{tabular}}
\caption{\large CIFAR-10}\label{tab:cifar10 gan}
\end{subtable}
}
\caption{PIGAN generator, discriminator and classifier architectures.}
\label{tb:pigan models}
\end{table*}

\begin{table*}[!h]
\centering
\ra{1.2}
\resizebox{0.9\textwidth}{!}{%
\begin{tabular}{@{}rccc@{}}\toprule
& \multicolumn{1}{c}{Fashion-MNIST} & \phantom{abc}& \multicolumn{1}{c}{CIFAR-10}   \\ \midrule
Inputs &  $\mathbf{x}\sim \mathbb{R}^{28\times28\times 1}$ and $\mathbf{y}\in \{1,\dots,10\}$ 
&& $\mathbf{x}\sim \mathbb{R}^{32\times32\times 3}$ and $\mathbf{y}\in\{1,\dots,10\}$  
\vspace{0.15cm}
\\
Layers  
& $3\times3$ stride=$1$ Conv $32$, ReLU
&&  $3\times3$ stride=$1$ Conv $32$, ReLU \\
& $3\times3$ stride=$1$ Conv $64$, ReLU, MaxPool, Dropout($0.5$)
&&  $3\times3$ stride=$1$ Conv $32$, ReLU, MaxPool, Dropout($0.2$)  \\
& $3\times3$ stride=$1$ Conv $128$, ReLU, MaxPool, Dropout($0.5$), Flatten 
&& $3\times3$ stride=$1$ Conv $64$, ReLU\\
& Dense $128$, ReLU, Dropout($0.5$) 
&& $3\times3$ stride=$1$ Conv $64$, ReLU, MaxPool, Dropout($0.2$) \\
& Dense $10$, Softmax
&&$3\times3$ stride=$1$ Conv $128$, ReLU\\
&  
&& $3\times3$ stride=$1$ Conv $128$, ReLU, MaxPool, Dropout($0.3$), Flatten \\
&  
&& Dense $256$, ReLU, Dropout($0.3$) \\
& 
&& Dense $10$, Softmax\\
\bottomrule
\end{tabular}}
\caption{Classifier architectures used to evaluate generative models in downstream tasks.}
\label{tb:downstream}
\end{table*}

\section{Additional Results}\label{app:additional}

\subsection{Further evaluation of PIGAN trained with different $\boldsymbol{\lambda}$}\label{app:lambda}
Fig.~\ref{fig:performance-rest} provides additional quantitative evaluation measures for PIGAN that were not presented in the main paper as the regularization parameter $\lambda$ increases. As expected, the TVD attack score, which is an upper limit on WB attack accuracy using the discriminator, is lower for larger values of $\lambda$ for both datasets. Similarly, PIGAN becomes less vulnerable to MC-Single attacks as $\lambda$ increases. Our observations from Fig.~\ref{fig:performance-rest} regarding MC-Single attacks  are aligned with the experiments in \cite{hilprecht2019monte} (also noted in \cite{mukherjee2021privgan}) that MC-Single attacks are much less successful compared to MC-Set attacks and achieve accuracy close to random guessing. As $\lambda$ is increased, the FID score also increases while the inception score decreases for Fashion-MNIST, indicating the degradation in generated sample quality. When moving from $\lambda=0$ to $\lambda=1$, an initial dip in FID and rise in  the Inception score is observed, which is in agreement with the small increase in classification accuracy observed in Fig.~\ref{fig:performance}, confirming an improvement in generated data fidelity. A possible explanation for this improvement for very small $\lambda$ could be that the PIGAN regularization used to prevent overfitting and improve privacy is also improving the training process. The curves for CIFAR-10 exhibit similar expected trends. 
In Fig.~\ref{fig:samples lambda} we visually compare the  samples generated with PIGAN and observe that while for small $\lambda$ the images are comparable to non-private GAN samples, the quality gets progressively worse with larger values of $\lambda$.

\begin{figure*}[!htb]
\centering
\begin{subfigure}{.9\textwidth}
  \centering
  \includegraphics[width=1\textwidth]{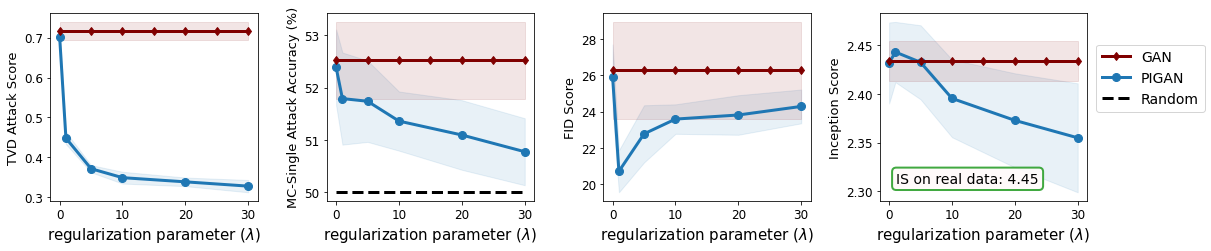} 
  \vspace{-6mm}
  \caption{Fashion-MNIST}
  \label{fig:sub-first}
\end{subfigure}
\\
\begin{subfigure}{.9\textwidth}
  \centering
  \includegraphics[width=1\textwidth]{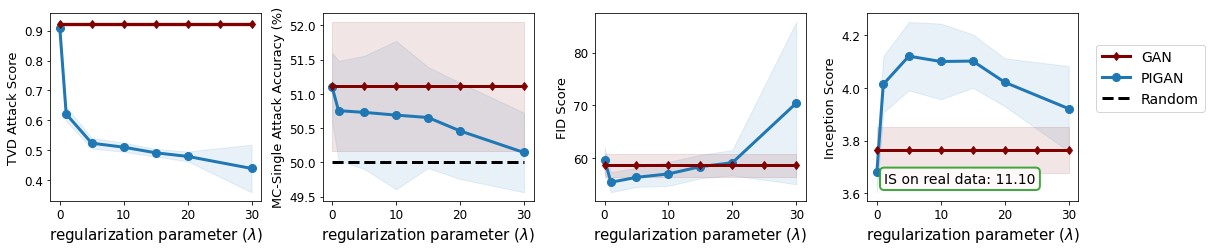}
    \vspace{-6mm}
  \caption{CIFAR 10}
  \label{fig:sub-second}
\end{subfigure}
\vspace{-2mm}
\caption{Additional privacy and fidelity measures for PIGAN trained with $N = 2$ and various regularization $\lambda$.}
\label{fig:performance-rest}
\end{figure*}

\begin{figure*}[!htb]\centering
\begin{subfigure}{.23\textwidth}\centering
  \includegraphics[width=1\textwidth]{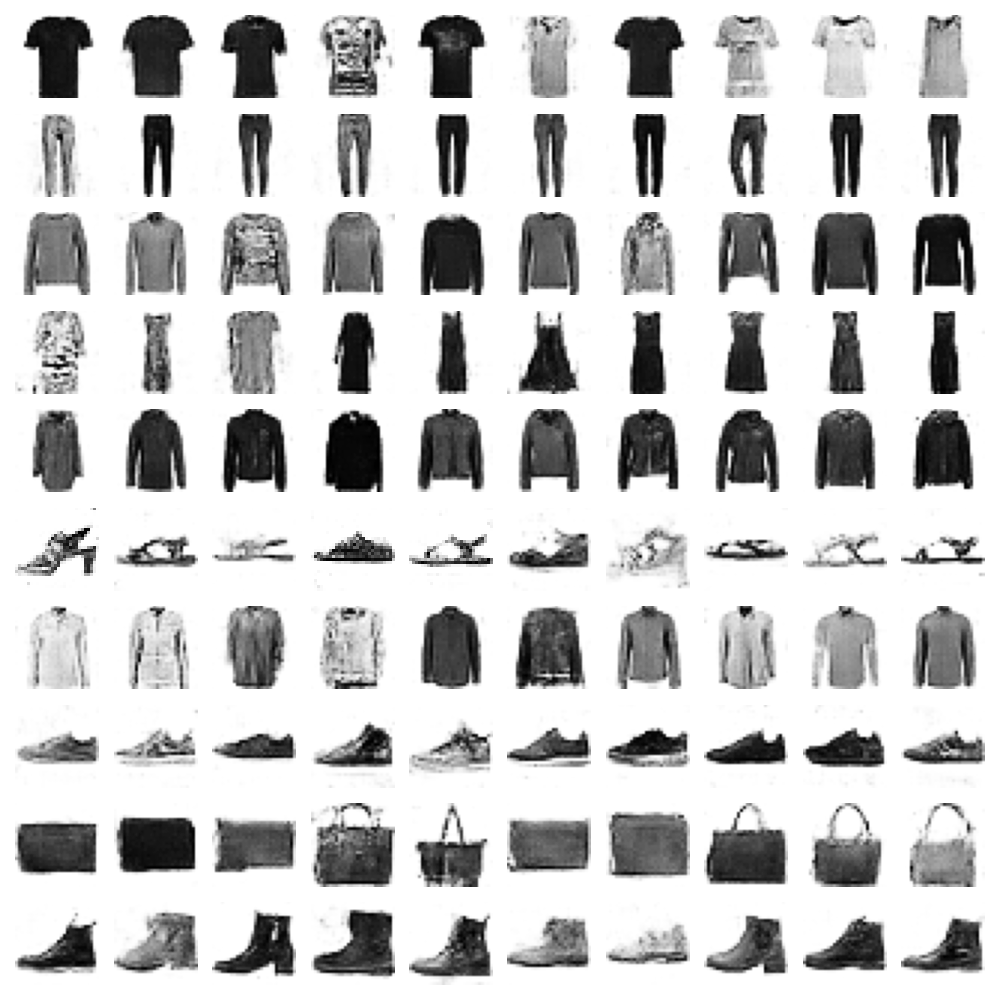} 
\end{subfigure}
\begin{subfigure}{.23\textwidth}\centering
  \includegraphics[width=1\textwidth]{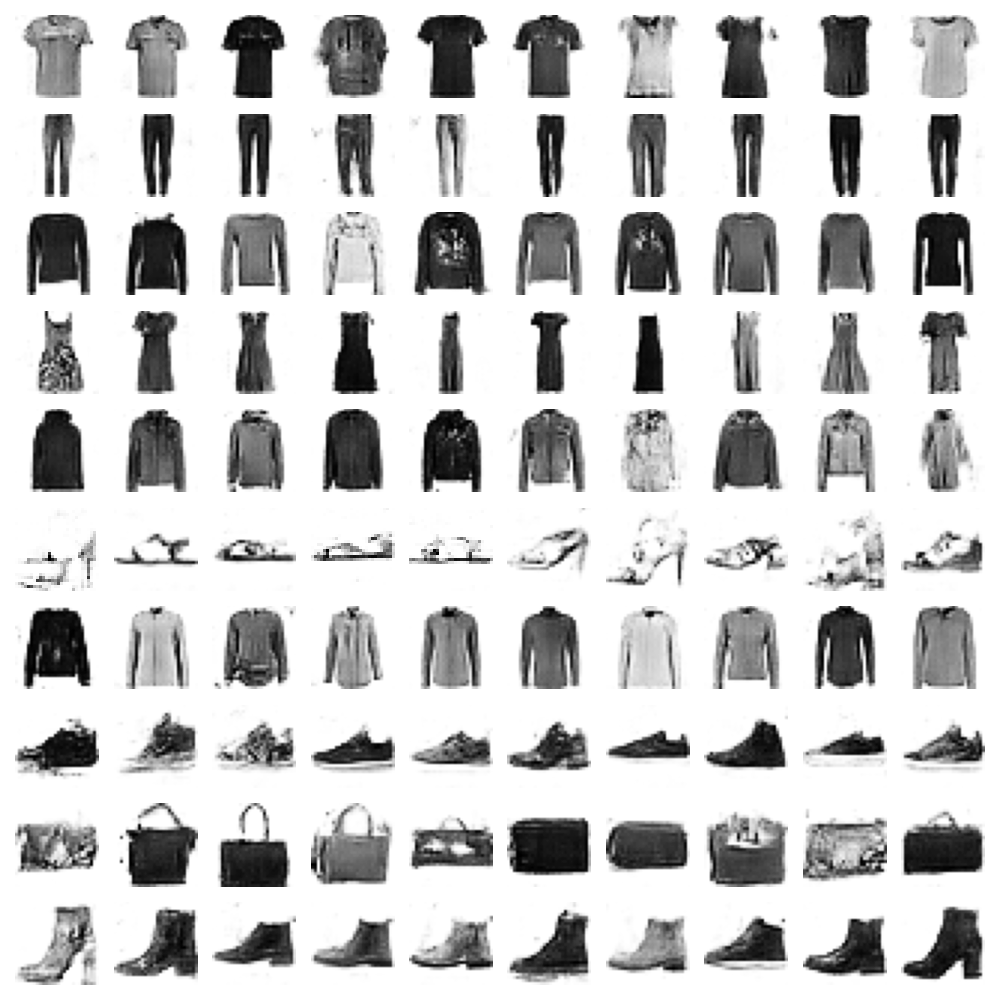}
\end{subfigure}
\begin{subfigure}{.23\textwidth}\centering
  \includegraphics[width=1\textwidth]{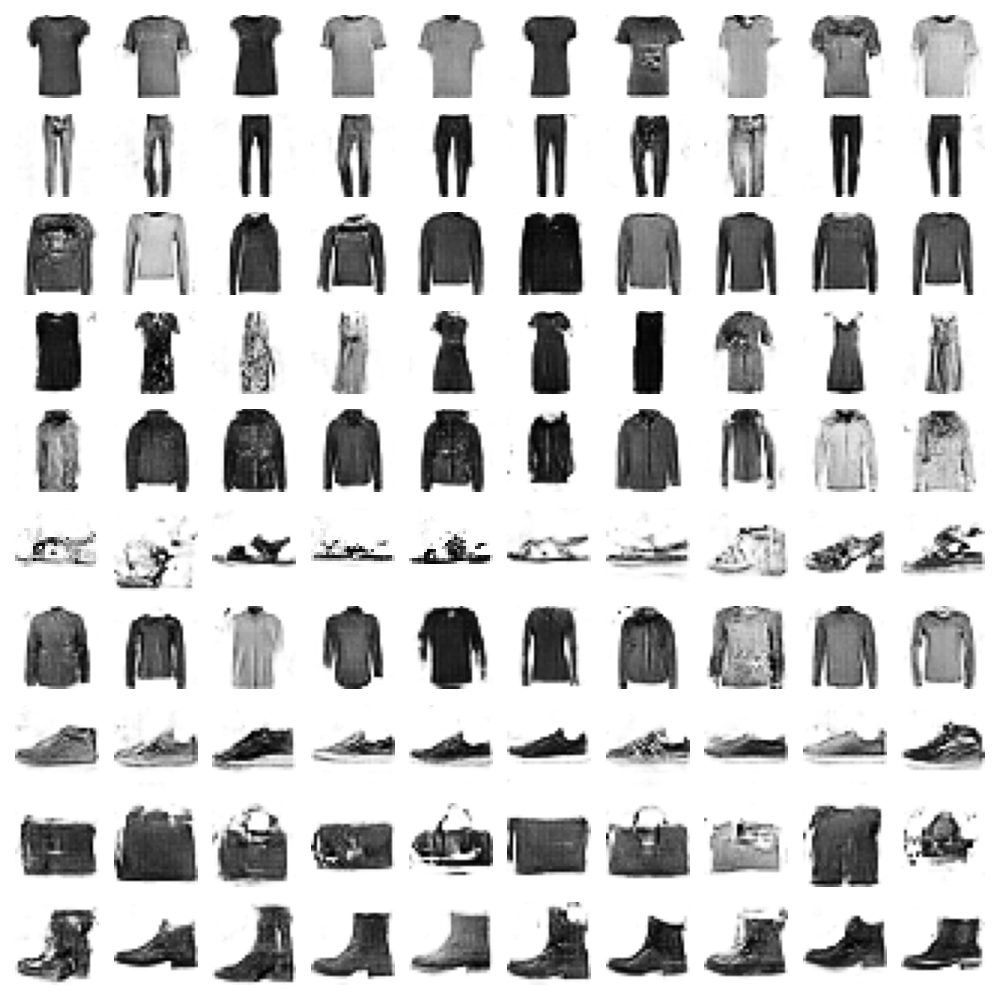} 
\end{subfigure}
\begin{subfigure}{.23\textwidth}\centering
  \includegraphics[width=1\textwidth]{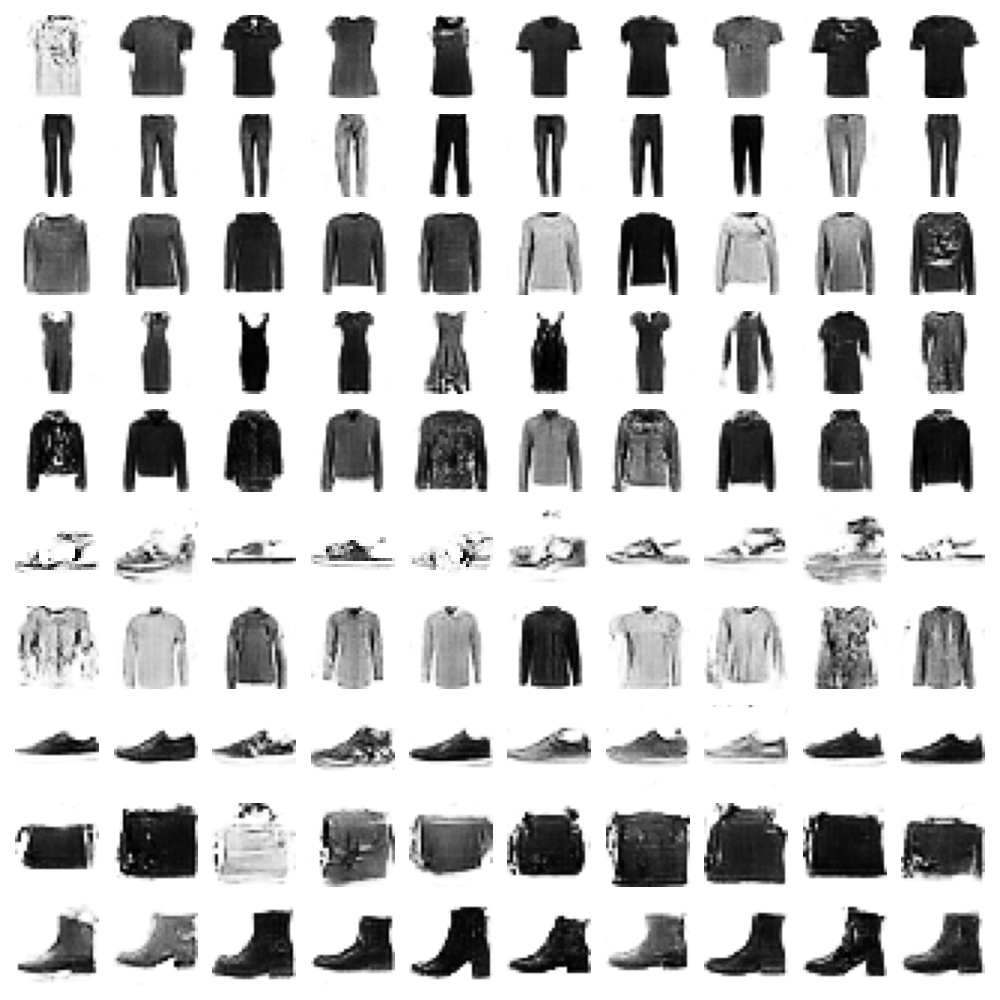}
\end{subfigure}
\\
\begin{subfigure}{.23\textwidth}\centering
  \includegraphics[width=1\textwidth]{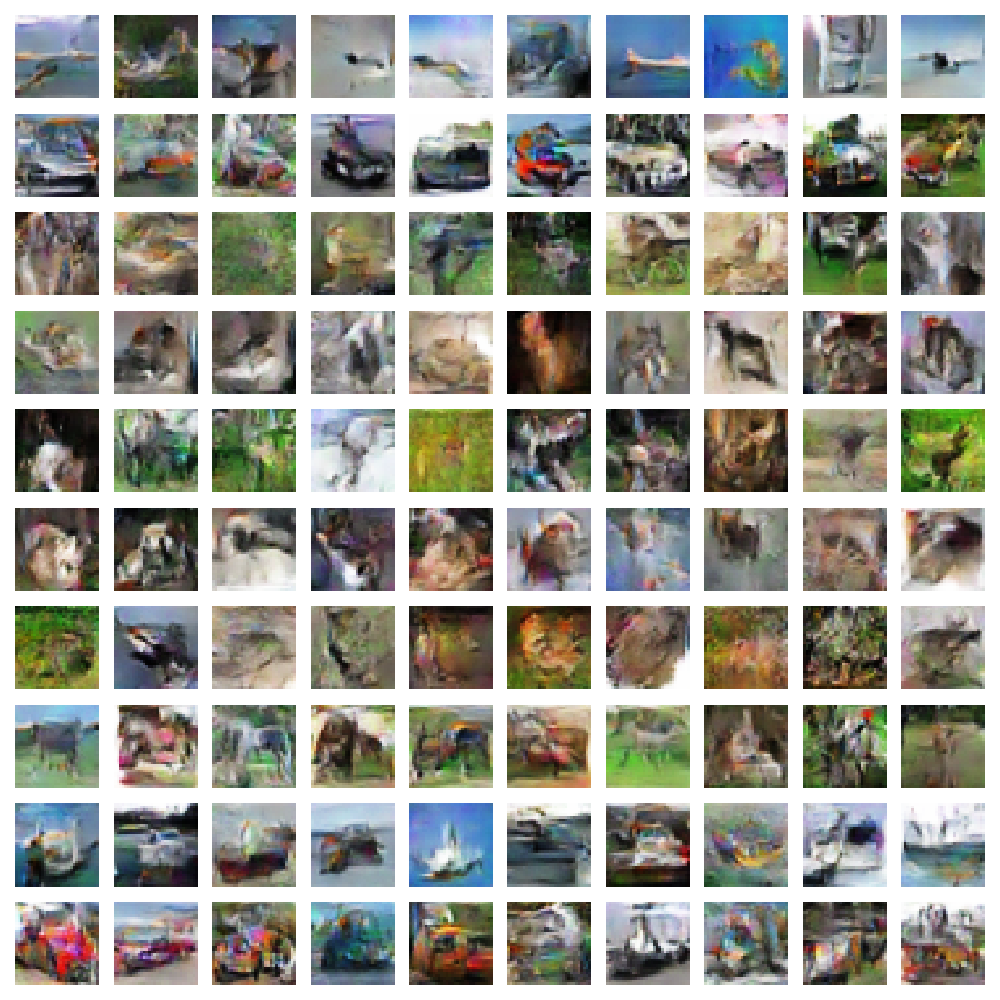}   \caption{GAN}
\end{subfigure}
\begin{subfigure}{.23\textwidth}\centering
  \includegraphics[width=1\textwidth]{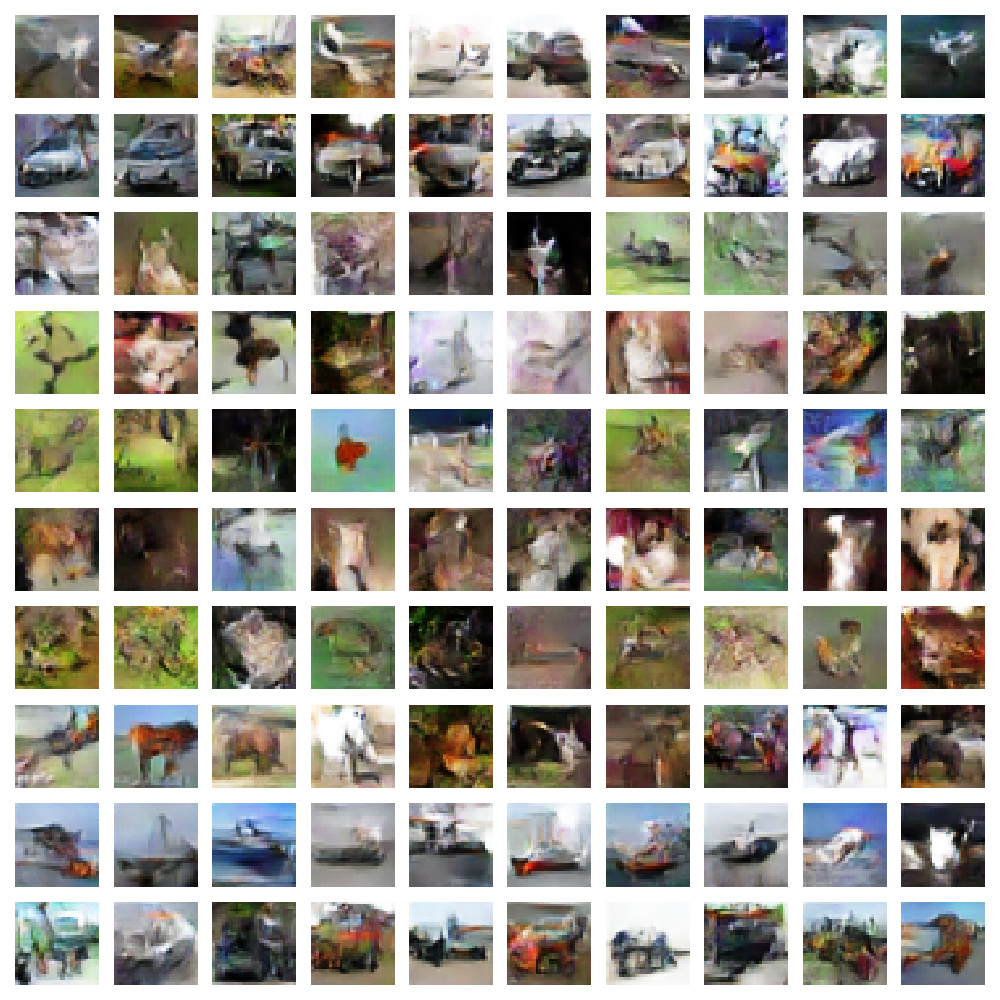}
  \caption{PIGAN ($\lambda=1$)}
\end{subfigure}
\begin{subfigure}{.23\textwidth}\centering
  \includegraphics[width=1\textwidth]{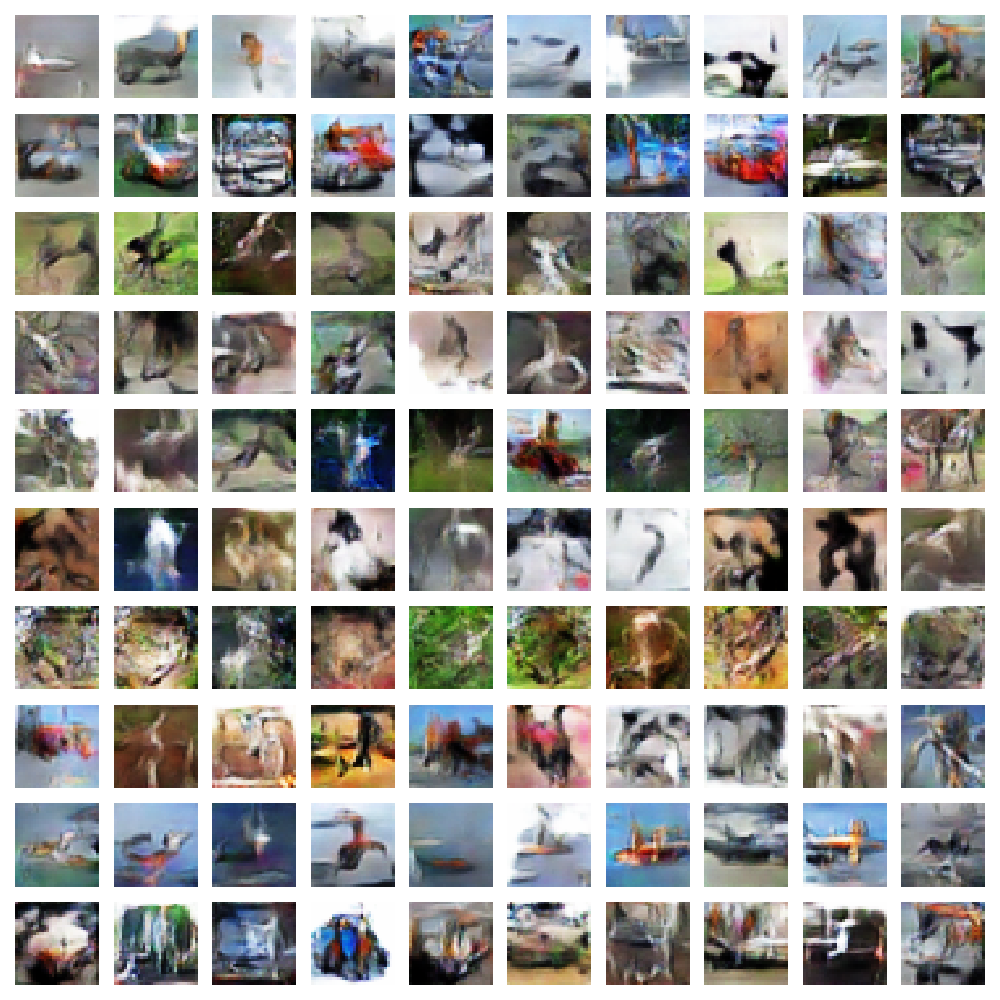} 
 \caption{PIGAN ($\lambda=10$)}
\end{subfigure}
\begin{subfigure}{.23\textwidth}\centering
  \includegraphics[width=1\textwidth]{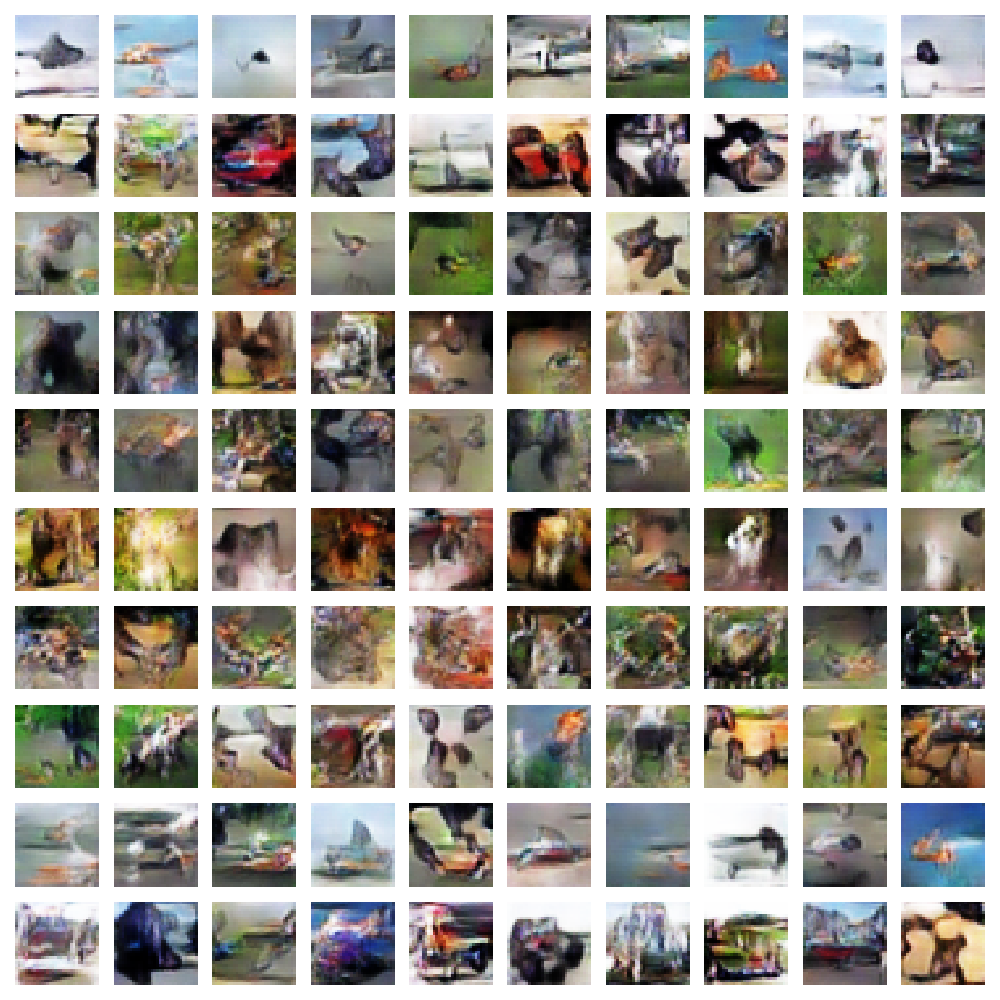}
 \caption{PIGAN ($\lambda=20$)}
\end{subfigure}

\vspace{-2mm}
\caption{Samples generated by PIGAN with various $\lambda$ compared to non-private GAN.}
\label{fig:samples lambda}
\end{figure*}

 \subsection{Comparison of the privacy-fidelity trade-off}\label{app:tradeoff}
 In Fig.~\ref{fig:tradeoff-rest}, we compare the different private models in terms of privacy and fidelity measures not presented in the main paper. As with the trade-offs observed in Fig.~\ref{fig:tradeoff}, the two left curves for both Fashion-MNIST and CIFAR-10 are almsot always higher than the curves of other private methods, and the two right curves are lower. It is also observed that PIGAN covers a wider privacy-level range compared to PrivGAN and especially DPGAN. For a given WB attack success rate, PIGAN is able to generate images with higher Inception score, i.e., better sample quality. Similarly, for a given MC-Set and MC-Single attack success rate, PIGAN results in better downstream classification performance and achieves lower FID and Intra-FID scores, respectively.  
 It is worth noting that as discussed in  \cite{borji2019pros} and \cite{liu2018improved}, both IS and FID score have limitations especially when evaluating class-conditional models. While we use Intra-FID to evaluate generated image quality in our experiments, other metrics for conditional generative models include Class-Aware Fr\'echet Distance (CAFD) \cite{liu2018improved}, Fr\'echet Joint Distance (FJD) \cite{devries2019evaluation} and Conditional IS and Conditional FID scores \cite{benny2021evaluation}.

 We visually compare the generated sample quality of different private models in Figs.~\ref{fig:fmnist compare} and \ref{fig:cifar10 compare}. For Fashion-MNIST it can be observed that PIGAN generates better samples visually mostly noticeable in categories such as sandals and bags. Samples generated with DPGAN have very poor quality especially in terms of color and item diversity. However, as expected, the discriminator distributions for DPGAN look almost identical confirming the strong privacy guarantees provided by differentially private trained models. Since we use class-conditional models in our experiments, we do not observe the mode-collapse resulting from high regularization values of $\lambda$ for PrivGAN, which was pointed out in \cite{mukherjee2021privgan}. From comparing the generated CIFAR-10 images in Fig.~\ref{fig:cifar10 compare}, we observe that despite the high regularization, PIGAN is able to generate images resembling those generated by non-private GAN in some of the categories, while PrivGAN generates images with lower quality overall. DPGAN is not able to generate images with reasonable visual quality, and as in Fashion-MNIST there is much less diversity in terms of color and objects within each class. 
\setlength{\textfloatsep}{1pt}
 \begin{figure*}[!hb]
\centering
\begin{subfigure}{.98\textwidth}
  \centering
  \includegraphics[width=1\textwidth]{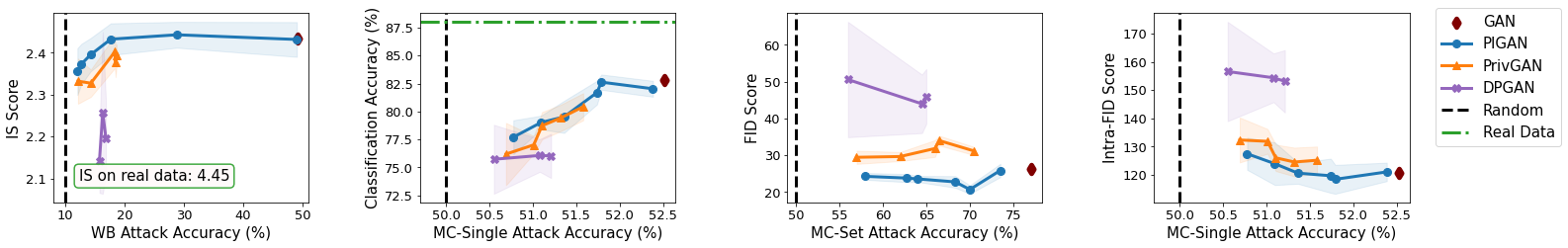} 
  \vspace{-5mm}
  \caption{Fashion-MNIST}
  \label{fig:sub-first}
\end{subfigure}
\\
\begin{subfigure}{.98\textwidth}
  \centering
  \includegraphics[width=1\textwidth]{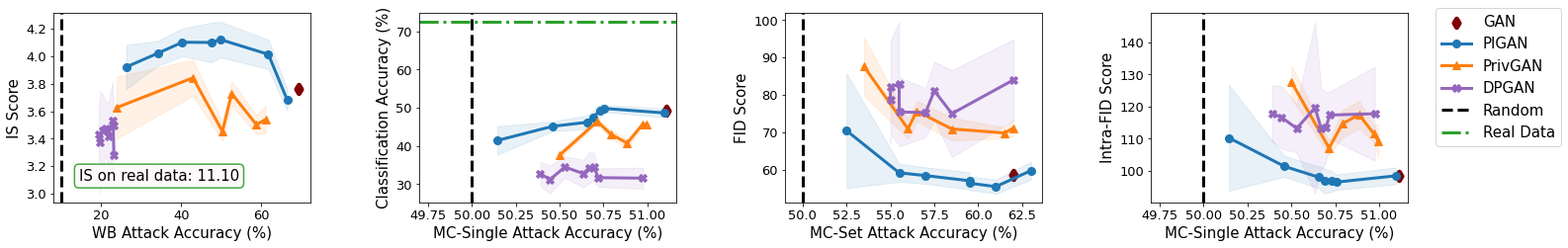}  
  \vspace{-5mm}
  \caption{CIFAR-10}
  \label{fig:sub-second}
\end{subfigure}
\vspace{-2mm}
\caption{Privacy-fidelity trade-off  achieved with different private models.}
\label{fig:tradeoff-rest}
\vspace{-6mm}
\end{figure*}

\begin{figure*}[!h]\centering
\begin{subfigure}{.23\textwidth}\centering
  \includegraphics[width=1\textwidth]{figures/fmnist/fmnist_gan_40.png} 
\end{subfigure}
\begin{subfigure}{.23\textwidth}\centering
  \includegraphics[width=1\textwidth]{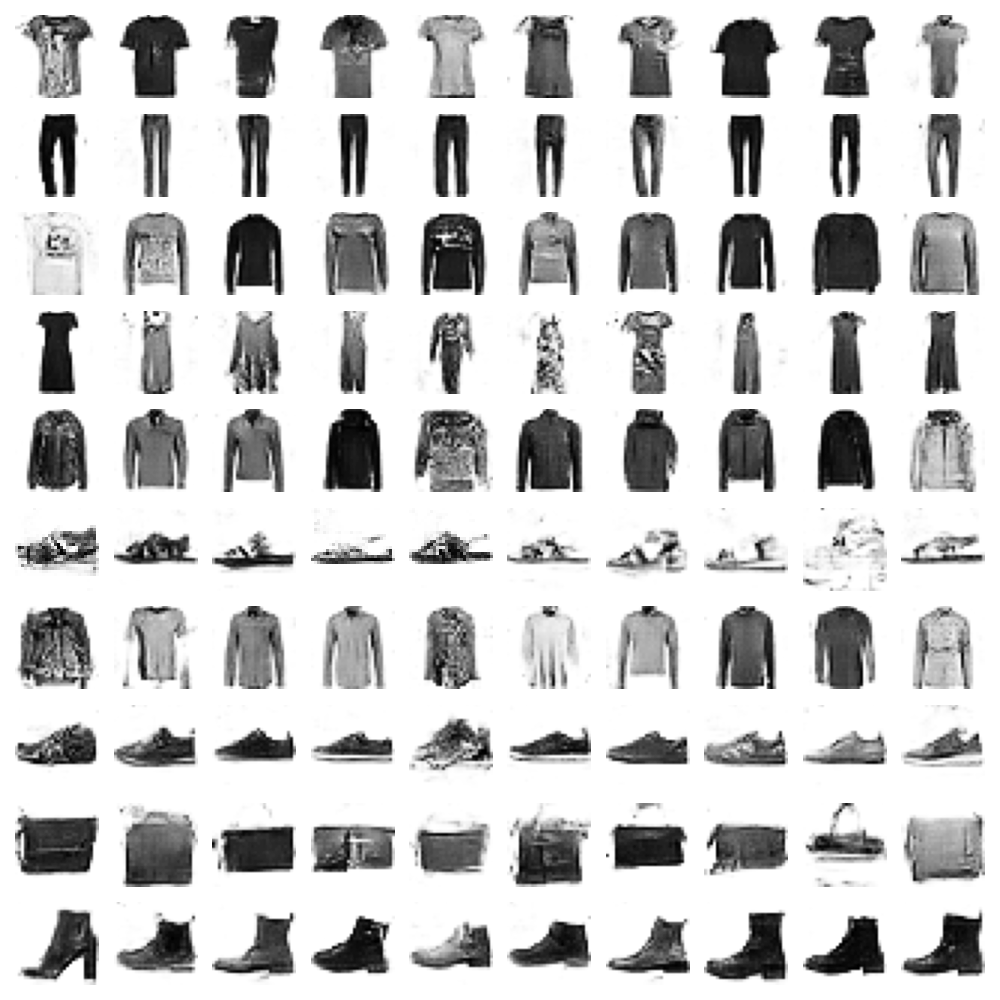}
\end{subfigure}
\begin{subfigure}{.23\textwidth}\centering
  \includegraphics[width=1\textwidth]{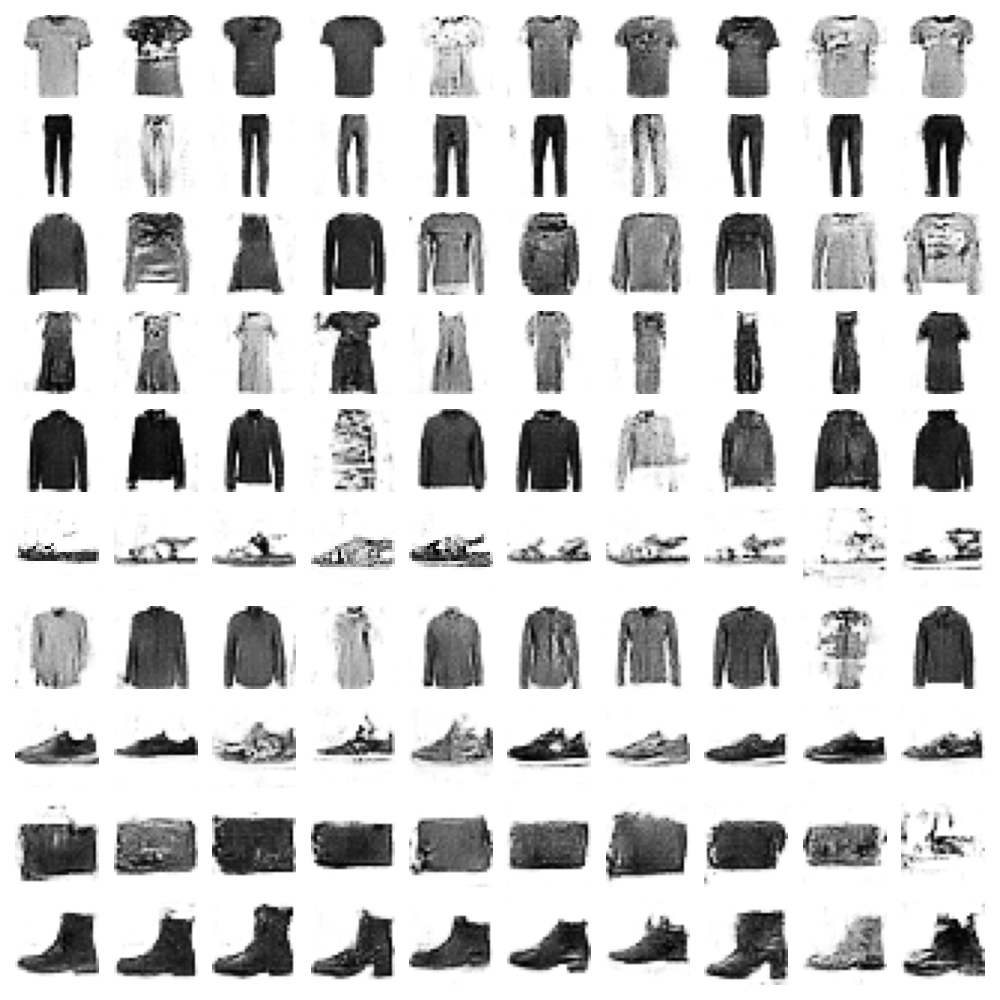} 
\end{subfigure}
\begin{subfigure}{.23\textwidth}\centering
  \includegraphics[width=1\textwidth]{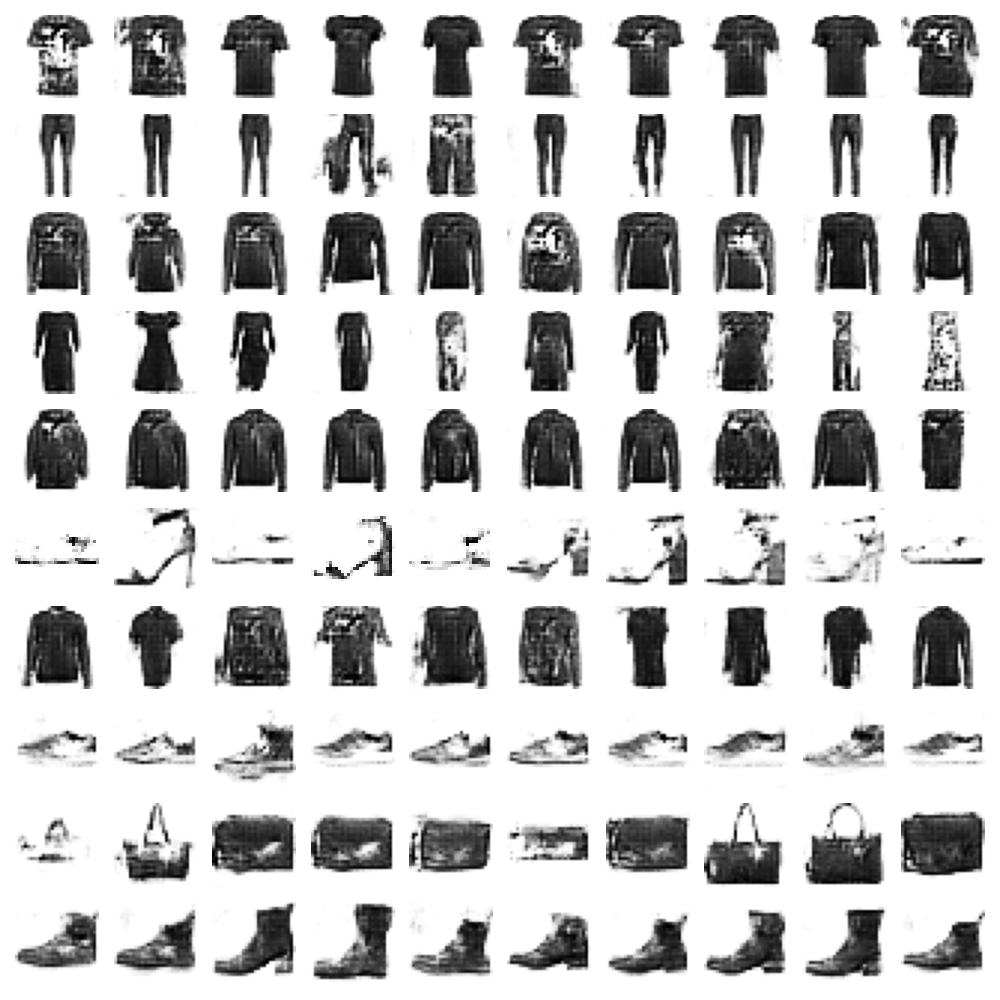}
\end{subfigure}
\\
\begin{subfigure}{.23\textwidth}\centering
  \includegraphics[width=1\textwidth]{figures/fmnist/fmnist_gan_40_dist.png}   \caption{GAN}
\end{subfigure}
\begin{subfigure}{.23\textwidth}\centering
  \includegraphics[width=1\textwidth]{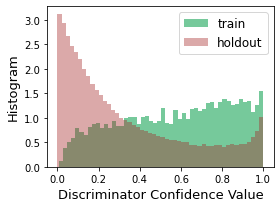}
  \caption{PIGAN ($\lambda=5$)}
\end{subfigure}
\begin{subfigure}{.23\textwidth}\centering
  \includegraphics[width=1\textwidth]{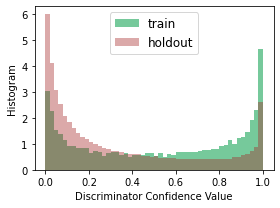} 
 \caption{PrivGAN ($\lambda=0.1$)}
\end{subfigure}
\begin{subfigure}{.23\textwidth}\centering
  \includegraphics[width=1\textwidth]{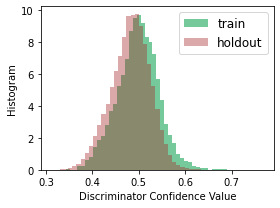}
 \caption{DPGAN {\scriptsize ($\sigma^2=0.5, c_p=1.2$)}}
\end{subfigure}

\vspace{-2mm}
\caption{Fashion-MNIST: Comparison between private models in terms of generated sample quality and discriminator confidence score distribution. Models are trained to achieve similar WB attack accuracy reported in Table~\ref{tb:comparison}.}
\label{fig:fmnist compare}
\end{figure*}

\begin{figure*}[!h]\centering
\begin{subfigure}{.23\textwidth}\centering
  \includegraphics[width=1\textwidth]{figures/cifar10/cifar10_gan_50.png} 
\end{subfigure}
\begin{subfigure}{.23\textwidth}\centering
  \includegraphics[width=1\textwidth]{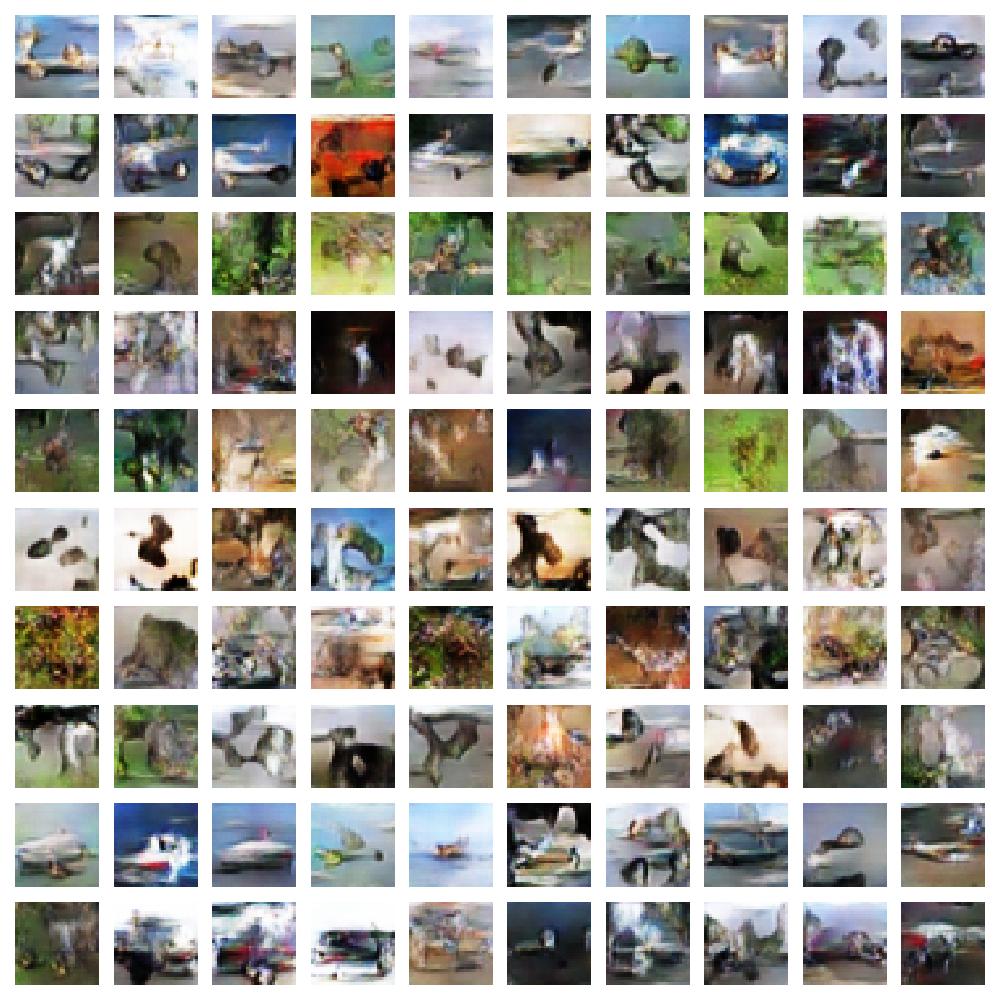}
\end{subfigure}
\begin{subfigure}{.23\textwidth}\centering
  \includegraphics[width=1\textwidth]{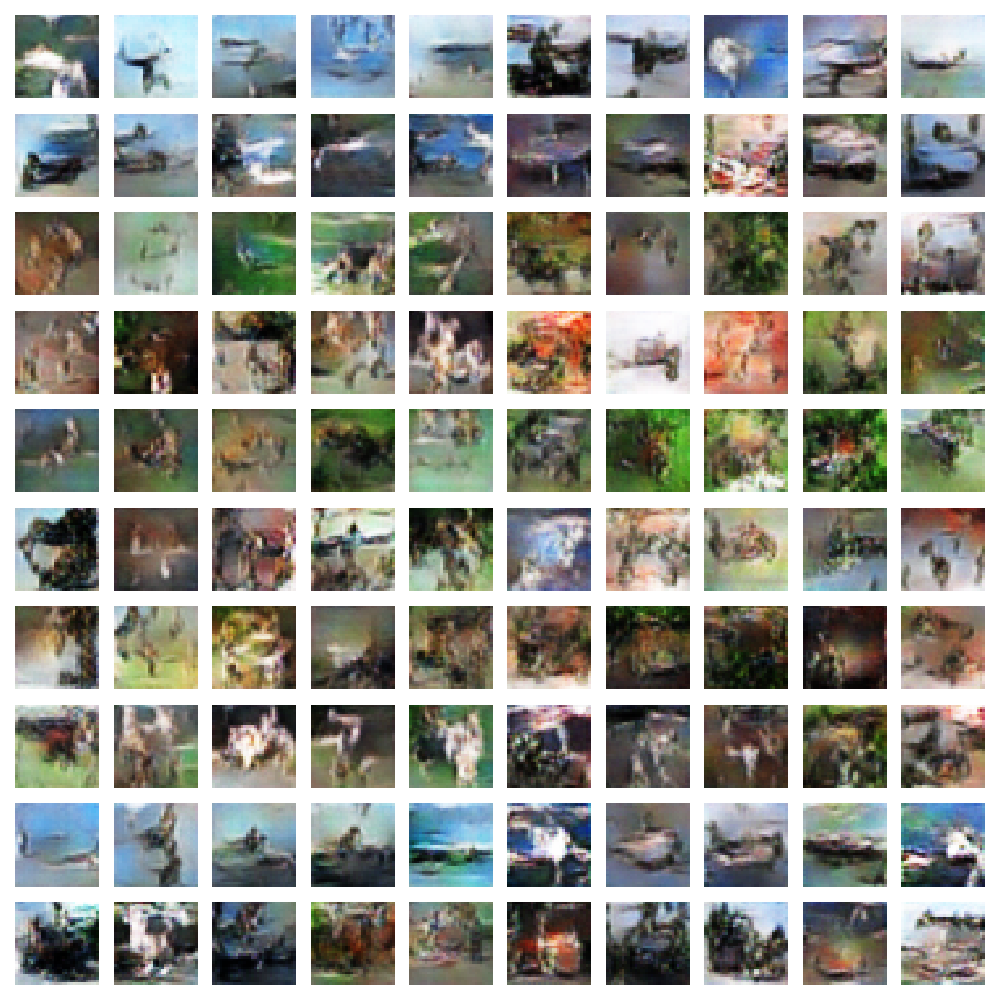} 
\end{subfigure}
\begin{subfigure}{.23\textwidth}\centering
  \includegraphics[width=1\textwidth]{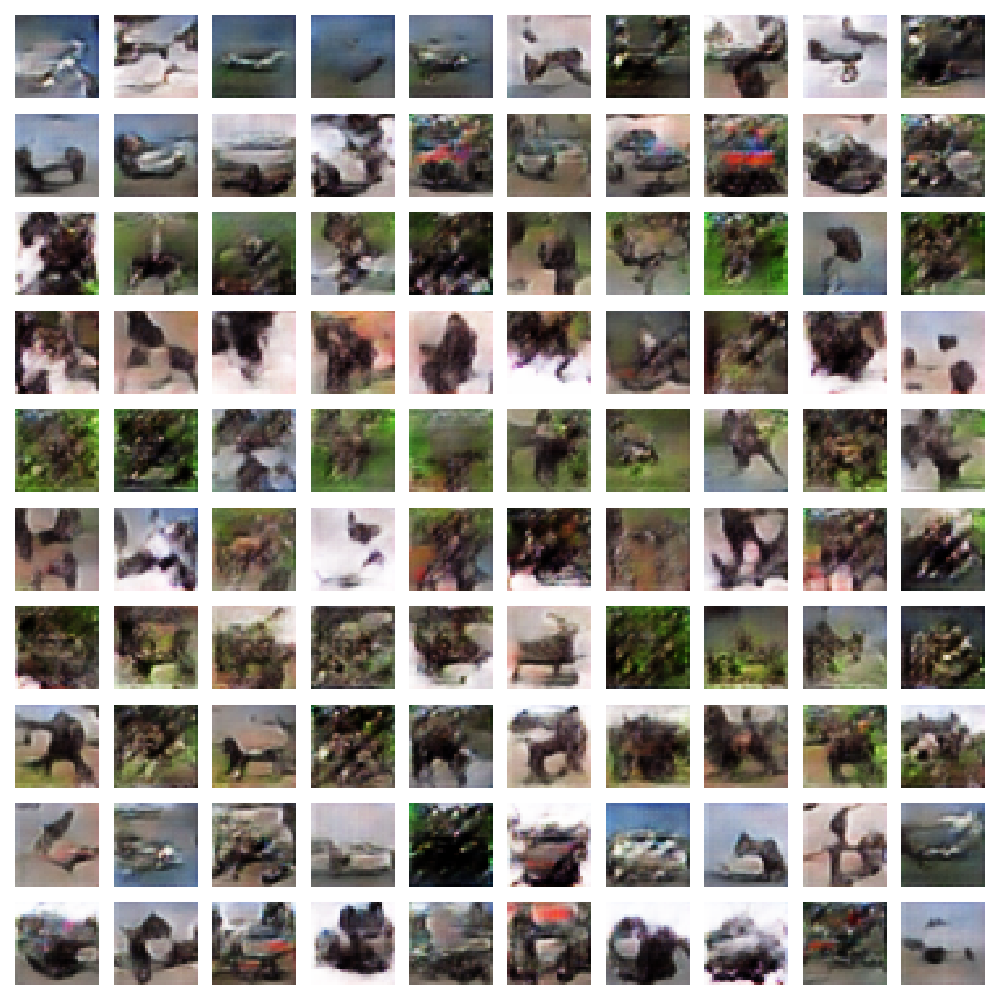}
\end{subfigure}
\\
\begin{subfigure}{.23\textwidth}\centering
  \includegraphics[width=1\textwidth]{figures/cifar10/cifar10_gan_50_dist.png}   \caption{GAN}
\end{subfigure}
\begin{subfigure}{.23\textwidth}\centering
  \includegraphics[width=1\textwidth]{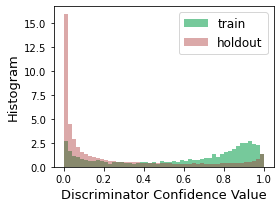}
  \caption{PIGAN ($\lambda=30$)}
\end{subfigure}
\begin{subfigure}{.23\textwidth}\centering
  \includegraphics[width=1\textwidth]{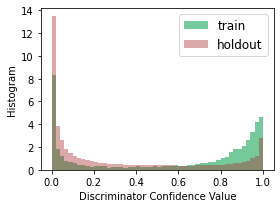} 
 \caption{PrivGAN ($\lambda=20$)}
\end{subfigure}
\begin{subfigure}{.23\textwidth}\centering
  \includegraphics[width=1\textwidth]{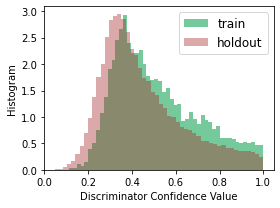}
 \caption{DPGAN {\scriptsize ($\sigma^2=0.7, c_p=1.5$)}}
\end{subfigure}

\vspace{-2mm}
\caption{CIFAR-10: Comparison between private models in terms of generated sample quality and discriminator confidence score distribution. Models are trained to achieve similar WB attack accuracy reported in Table~\ref{tb:comparison}.}
\label{fig:cifar10 compare}
\end{figure*}

%% file: main.bbl
\begin{thebibliography}{32}
\providecommand{\natexlab}[1]{#1}
\providecommand{\url}[1]{\texttt{#1}}
\expandafter\ifx\csname urlstyle\endcsname\relax
  \providecommand{\doi}[1]{doi: #1}\else
  \providecommand{\doi}{doi: \begingroup \urlstyle{rm}\Url}\fi

\bibitem[Arjovsky et~al.(2017)Arjovsky, Chintala, and
  Bottou]{arjovsky2017wasserstein}
Arjovsky, M., Chintala, S., and Bottou, L.
\newblock Wasserstein generative adversarial networks.
\newblock In \emph{International conference on machine learning}, pp.\
  214--223. PMLR, 2017.

\bibitem[Benny et~al.(2021)Benny, Galanti, Benaim, and
  Wolf]{benny2021evaluation}
Benny, Y., Galanti, T., Benaim, S., and Wolf, L.
\newblock Evaluation metrics for conditional image generation.
\newblock \emph{International Journal of Computer Vision}, 129\penalty0
  (5):\penalty0 1712--1731, 2021.

\bibitem[Borji(2019)]{borji2019pros}
Borji, A.
\newblock Pros and cons of gan evaluation measures.
\newblock \emph{Computer Vision and Image Understanding}, 179:\penalty0 41--65,
  2019.

\bibitem[Brock et~al.(2019)Brock, Donahue, and Simonyan]{brock2018large}
Brock, A., Donahue, J., and Simonyan, K.
\newblock Large scale {GAN} training for high fidelity natural image synthesis.
\newblock In \emph{7th International Conference on Learning Representations,
  ICLR 2019}, 2019.

\bibitem[Chen et~al.(2020{\natexlab{a}})Chen, Orekondy, and Fritz]{chen2020gs}
Chen, D., Orekondy, T., and Fritz, M.
\newblock Gs-wgan: A gradient-sanitized approach for learning differentially
  private generators.
\newblock \emph{arXiv preprint arXiv:2006.08265}, 2020{\natexlab{a}}.

\bibitem[Chen et~al.(2020{\natexlab{b}})Chen, Yu, Zhang, and
  Fritz]{chen2020gan}
Chen, D., Yu, N., Zhang, Y., and Fritz, M.
\newblock Gan-leaks: A taxonomy of membership inference attacks against
  generative models.
\newblock In \emph{Proceedings of the 2020 ACM SIGSAC conference on computer
  and communications security}, pp.\  343--362, 2020{\natexlab{b}}.

\bibitem[Chen et~al.(2016)Chen, Duan, Houthooft, Schulman, Sutskever, and
  Abbeel]{chen2016infogan}
Chen, X., Duan, Y., Houthooft, R., Schulman, J., Sutskever, I., and Abbeel, P.
\newblock Infogan: Interpretable representation learning by information
  maximizing generative adversarial nets.
\newblock In \emph{Proceedings of the 30th International Conference on Neural
  Information Processing Systems}, pp.\  2180--2188, 2016.

\bibitem[DeVries et~al.(2019)DeVries, Romero, Pineda, Taylor, and
  Drozdzal]{devries2019evaluation}
DeVries, T., Romero, A., Pineda, L., Taylor, G.~W., and Drozdzal, M.
\newblock On the evaluation of conditional gans.
\newblock \emph{arXiv preprint arXiv:1907.08175}, 2019.

\bibitem[Goodfellow et~al.(2014)Goodfellow, Pouget-Abadie, Mirza, Xu,
  Warde-Farley, Ozair, Courville, and Bengio]{goodfellow2014generative}
Goodfellow, I., Pouget-Abadie, J., Mirza, M., Xu, B., Warde-Farley, D., Ozair,
  S., Courville, A., and Bengio, Y.
\newblock Generative adversarial nets.
\newblock \emph{Advances in neural information processing systems}, 27, 2014.

\bibitem[Hayes et~al.(2019)Hayes, Melis, Danezis, and
  De~Cristofaro]{hayes2019logan}
Hayes, J., Melis, L., Danezis, G., and De~Cristofaro, E.
\newblock Logan: Membership inference attacks against generative models.
\newblock \emph{Proceedings on Privacy Enhancing Technologies}, 2019\penalty0
  (1):\penalty0 133--152, 2019.

\bibitem[Heusel et~al.(2017)Heusel, Ramsauer, Unterthiner, Nessler, and
  Hochreiter]{heusel2017gans}
Heusel, M., Ramsauer, H., Unterthiner, T., Nessler, B., and Hochreiter, S.
\newblock Gans trained by a two time-scale update rule converge to a local nash
  equilibrium.
\newblock \emph{Advances in neural information processing systems}, 30, 2017.

\bibitem[Hilprecht et~al.(2019)Hilprecht, H{\"a}rterich, and
  Bernau]{hilprecht2019monte}
Hilprecht, B., H{\"a}rterich, M., and Bernau, D.
\newblock Monte carlo and reconstruction membership inference attacks against
  generative models.
\newblock \emph{Proceedings on Privacy Enhancing Technologies}, 2019\penalty0
  (4):\penalty0 232--249, 2019.

\bibitem[Hinton \& Salakhutdinov(2006)Hinton and
  Salakhutdinov]{hinton2006reducing}
Hinton, G.~E. and Salakhutdinov, R.~R.
\newblock Reducing the dimensionality of data with neural networks.
\newblock \emph{science}, 313\penalty0 (5786):\penalty0 504--507, 2006.

\bibitem[Hitaj et~al.(2017)Hitaj, Ateniese, and Perez-Cruz]{hitaj2017deep}
Hitaj, B., Ateniese, G., and Perez-Cruz, F.
\newblock Deep models under the gan: information leakage from collaborative
  deep learning.
\newblock In \emph{Proceedings of the 2017 ACM SIGSAC Conference on Computer
  and Communications Security}, pp.\  603--618, 2017.

\bibitem[Hoang et~al.(2018)Hoang, Nguyen, Le, and Phung]{hoang2018mgan}
Hoang, Q., Nguyen, T.~D., Le, T., and Phung, D.
\newblock Mgan: Training generative adversarial nets with multiple generators.
\newblock In \emph{International conference on learning representations}, 2018.

\bibitem[Jordon et~al.(2018)Jordon, Yoon, and Van Der~Schaar]{jordon2018pate}
Jordon, J., Yoon, J., and Van Der~Schaar, M.
\newblock Pate-gan: Generating synthetic data with differential privacy
  guarantees.
\newblock In \emph{International conference on learning representations}, 2018.

\bibitem[Kingma \& Ba(2014)Kingma and Ba]{kingma2014adam}
Kingma, D.~P. and Ba, J.
\newblock Adam: A method for stochastic optimization.
\newblock \emph{arXiv preprint arXiv:1412.6980}, 2014.

\bibitem[Kingma \& Welling(2013)Kingma and Welling]{kingma2013auto}
Kingma, D.~P. and Welling, M.
\newblock Auto-encoding variational bayes.
\newblock In \emph{Proceedings of the 2th International Conference on Learning
  Representations (ICLR)}, 2013.

\bibitem[Krizhevsky et~al.(2009)Krizhevsky, Hinton,
  et~al.]{krizhevsky2009learning}
Krizhevsky, A., Hinton, G., et~al.
\newblock Learning multiple layers of features from tiny images.
\newblock 2009.

\bibitem[Liu et~al.(2018)Liu, Wei, Lu, and Zhou]{liu2018improved}
Liu, S., Wei, Y., Lu, J., and Zhou, J.
\newblock An improved evaluation framework for generative adversarial networks.
\newblock \emph{arXiv preprint arXiv:1803.07474}, 2018.

\bibitem[Liu et~al.(2019)Liu, Peng, James, and Wu]{liu2019ppgan}
Liu, Y., Peng, J., James, J., and Wu, Y.
\newblock Ppgan: Privacy-preserving generative adversarial network.
\newblock In \emph{2019 IEEE 25Th international conference on parallel and
  distributed systems (ICPADS)}, pp.\  985--989. IEEE, 2019.

\bibitem[Mirza \& Osindero(2014)Mirza and Osindero]{mirza2014conditional}
Mirza, M. and Osindero, S.
\newblock Conditional generative adversarial nets.
\newblock \emph{arXiv preprint arXiv:1411.1784}, 2014.

\bibitem[Miyato \& Koyama(2018)Miyato and Koyama]{miyato2018cgans}
Miyato, T. and Koyama, M.
\newblock cgans with projection discriminator.
\newblock \emph{arXiv preprint arXiv:1802.05637}, 2018.

\bibitem[Mukherjee et~al.(2021)Mukherjee, Xu, Trivedi, Patowary, and
  Ferres]{mukherjee2021privgan}
Mukherjee, S., Xu, Y., Trivedi, A., Patowary, N., and Ferres, J.~L.
\newblock privgan: Protecting gans from membership inference attacks at low
  cost to utility.
\newblock \emph{Proc. Priv. Enhancing Technol.}, 2021\penalty0 (3):\penalty0
  142--163, 2021.

\bibitem[Radford et~al.(2015)Radford, Metz, and
  Chintala]{radford2015unsupervised}
Radford, A., Metz, L., and Chintala, S.
\newblock Unsupervised representation learning with deep convolutional
  generative adversarial networks.
\newblock \emph{arXiv preprint arXiv:1511.06434}, 2015.

\bibitem[Ravuri \& Vinyals(2019)Ravuri and Vinyals]{ravuri2019classification}
Ravuri, S. and Vinyals, O.
\newblock Classification accuracy score for conditional generative models.
\newblock \emph{arXiv preprint arXiv:1905.10887}, 2019.

\bibitem[Salimans et~al.(2016)Salimans, Goodfellow, Zaremba, Cheung, Radford,
  and Chen]{salimans2016improved}
Salimans, T., Goodfellow, I., Zaremba, W., Cheung, V., Radford, A., and Chen,
  X.
\newblock Improved techniques for training gans.
\newblock \emph{Advances in neural information processing systems},
  29:\penalty0 2234--2242, 2016.

\bibitem[Shmelkov et~al.(2018)Shmelkov, Schmid, and Alahari]{shmelkov2018good}
Shmelkov, K., Schmid, C., and Alahari, K.
\newblock How good is my gan?
\newblock In \emph{Proceedings of the European Conference on Computer Vision
  (ECCV)}, pp.\  213--229, 2018.

\bibitem[Shokri et~al.(2017)Shokri, Stronati, Song, and
  Shmatikov]{shokri2017membership}
Shokri, R., Stronati, M., Song, C., and Shmatikov, V.
\newblock Membership inference attacks against machine learning models.
\newblock In \emph{2017 IEEE Symposium on Security and Privacy (SP)}, pp.\
  3--18. IEEE, 2017.

\bibitem[Szegedy et~al.(2016)Szegedy, Vanhoucke, Ioffe, Shlens, and
  Wojna]{szegedy2016rethinking}
Szegedy, C., Vanhoucke, V., Ioffe, S., Shlens, J., and Wojna, Z.
\newblock Rethinking the inception architecture for computer vision.
\newblock In \emph{Proceedings of the IEEE conference on computer vision and
  pattern recognition}, pp.\  2818--2826, 2016.

\bibitem[Xiao et~al.(2017)Xiao, Rasul, and Vollgraf]{xiao2017fashion}
Xiao, H., Rasul, K., and Vollgraf, R.
\newblock Fashion-mnist: a novel image dataset for benchmarking machine
  learning algorithms.
\newblock \emph{arXiv preprint arXiv:1708.07747}, 2017.

\bibitem[Xie et~al.(2018)Xie, Lin, Wang, Wang, and Zhou]{DPGAN}
Xie, L., Lin, K., Wang, S., Wang, F., and Zhou, J.
\newblock Differentially private generative adversarial network.
\newblock \emph{CoRR}, abs/1802.06739, 2018.

\end{thebibliography}
